\documentclass[sigconf, nonacm]{acmart}

\AtBeginDocument{%
  }

\setcopyright{none}
\usepackage{rotating}

\newcommand{\attacksetting}{PGD-250, $\ell_2$, $\epsilon=7.0$, step size $\alpha=0.3$, non-targeted}

\usepackage{mathtools}
\usepackage{bm}
\usepackage{enumitem}
\usepackage{hyperref}

\theoremstyle{definition}
\newtheorem{definition}{Definition}
\newtheorem{assumption}{Assumption}
\newtheorem{theorem}{Theorem}
\newtheorem{proposition}{Proposition}

\theoremstyle{remark}
\newtheorem{remark}{Remark}

\newcommand{\norm}[1]{\left\lVert #1 \right\rVert}
\newcommand{\inner}[2]{\left\langle #1,#2 \right\rangle}

\begin{document}

% \title{Why Reinforcement Learning Makes Gradient-Based Adversarial Attacks Less Effective: A Study of Exploration-Induced Robustness}
\title{Reinforcement Learning Disrupts Gradient-Based Adversarial Optimization}

% %
% % The "author" command and its associated commands are used to define
% % the authors and their affiliations.
% % Of note is the shared affiliation of the first two authors, and the
% % "authornote" and "authornotemark" commands
% % used to denote shared contribution to the research.
% \affiliation{%

\author{Xinhai Zou}
\authornote{The first three authors contributed equally to this research.}
\affiliation{%
  \institution{COSIC, KU Leuven}
  \city{Leuven}
  \country{Belgium}}
\email{xinhai.zou@esat.kuleuven.com}
\orcid{0009-0008-3703-1253}

\author{Chang Zhao}
\authornotemark[1]
\affiliation{%
  \institution{Imec}
  \city{Leuven}
  \country{Belgium}}
\affiliation{%
  \institution{Brubotics, VUB}
  \city{Brussels}
  \country{Belgium}
}
\email{changzhao@imec.be}
\orcid{0009-0007-7988-5807}

\author{Alireza Aghabagherloo}
\authornotemark[1]
\affiliation{%
  \institution{COSIC, KU Leuven}
  \city{Leuven}
  \country{Belgium}}
\email{alireza.aghabagherloo@esat.kuleuven.be}
\orcid{0000-0003-0783-0261}

\author{Dave Singelée}
\affiliation{%
  \institution{DistriNet, KU Leuven}
  \city{Leuven}
  \country{Belgium}}
\email{dave.singelee@kuleuven.be}
\orcid{0000-0001-9084-698X}

\author{Robin Degraeve}
\affiliation{%
  \institution{Imec}
  \city{Leuven}
  \country{Belgium}}
\email{Robin.degraeve@imec.be}
\orcid{0000-0002-4609-5573}

\author{Bart Preneel}
\affiliation{%
  \institution{COSIC, KU Leuven}
  \city{Leuven}
  \country{Belgium}}
\email{bart.preneel@esat.kuleuven.be}
\orcid{0000-0003-2005-9651}

\begin{abstract}
    Gradient-based adversarial attacks remain a dominant threat to deep neural networks (DNNs), as they exploit gradient information to efficiently optimize adversarial perturbations. To address this, we investigate whether reinforcement learning (RL) training can disrupt the gradient structure used by attackers, by training image classifiers with policy-gradient objectives and $\varepsilon$-greedy exploration. Through systematic experiments across CIFAR-10, CIFAR-100, and ImageNet-100 with multiple architectures, we find that RL-trained classifiers significantly disrupt gradient-based adversarial optimization. For example, on a representative 6-layer CNN, adversarial accuracy under PGD on CIFAR-10 increases from 5\% (SL) to 56\% (RL), while clean accuracy decreases only 2–3\%. To explain this, we conduct a comprehensive mechanism analysis using loss landscape visualization, static and dynamic gradient indicators, and predictive entropy. Our analysis reveals that RL acts as an implicit regularizer, producing models with highly unstable gradient directions and smaller gradient magnitudes. This combination makes each PGD step both unreliable in direction and limited in magnitude, causing gradient-based attacks to fail within practical iteration budgets. We further show that combining RL with adversarial training (RL-adv) provides a dual-layer defense operating at two complementary levels: RL degrades gradient information available to attackers (gradient-level defense), while adversarial training strengthens decision boundaries (boundary-level defense). RL-adv achieves the highest robustness across all major attack types evaluated, including gradient-based (PGD, AutoAttack), transfer-based, and query-based attacks, outperforming SL-adv by a significant margin (e.g., 36.27\% vs. 24.87\% under APGD-CE on CIFAR-10). These findings identify RL-induced gradient disruption as a complementary robustness mechanism and motivate future research on hybrid SL-RL training schedules that combine SL's efficiency with RL's gradient-regularization properties.
\end{abstract}

\begin{teaserfigure}
  \includegraphics[width=\textwidth]{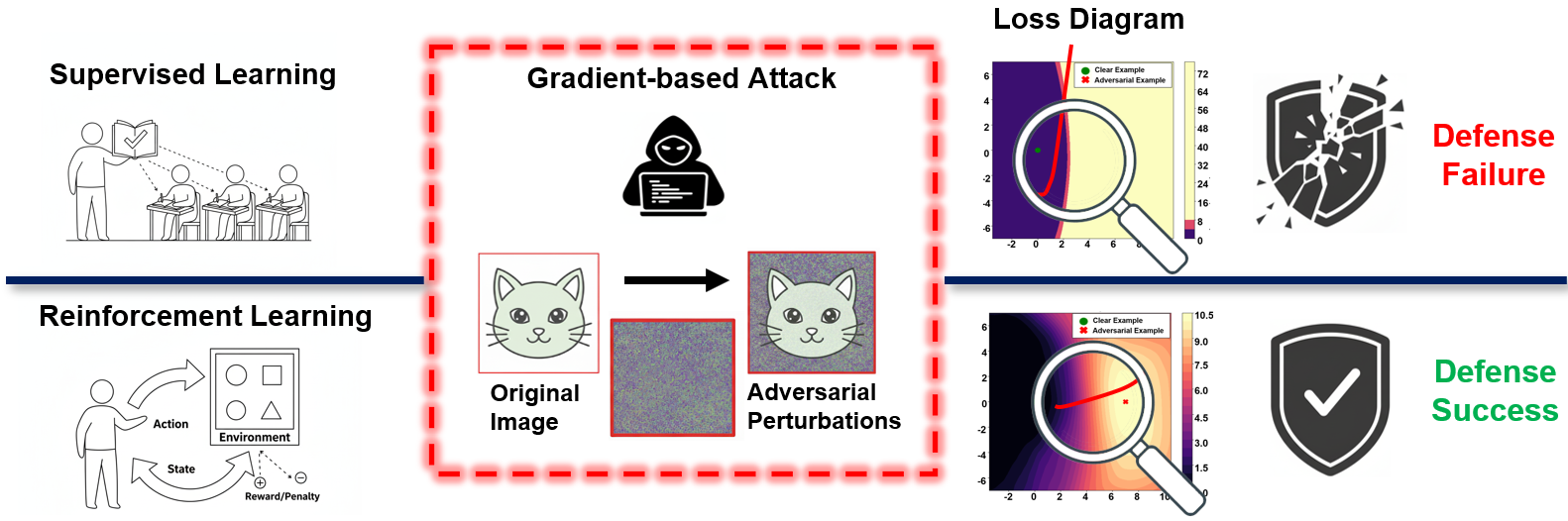}
  \caption{We compare \textbf{SL} and \textbf{RL} training for image classification under gradient-based adversarial attacks. \textbf{Top (SL):} SL-trained models retain sharper, more informative gradients, which adversaries can readily exploit. \textbf{Bottom (RL):} RL-trained models exhibit flatter, less informative gradients, offering no clear gradient direction for attack. This mechanism matches our empirical results on CIFAR-10/100 and ImageNet-100.}
  \label{fig:intro_fig}
\end{teaserfigure}

% \begin{figure}[htbp]
%     \centering
%     \includegraphics[width=0.85\linewidth]{figs/introduction_figure.png}
%     \caption{We compare \textbf{SL} and \textbf{RL} training for image classification under gradient-based adversarial attacks. \textbf{Top (SL):} SL-trained models retain sharper, more informative gradients, which adversaries can readily exploit. \textbf{Bottom (RL):} RL-trained models exhibit flatter, less informative gradients, offering no clear gradient direction for attack. This mechanism matches our empirical results on CIFAR-10/100 and ImageNet-100.}
%     \label{fig:intro_fig}
% \end{figure}


\maketitle

\section{Introduction}\label{sec:introduction} 
Deep neural networks (DNNs) are widely used in modern AI systems, but their vulnerability to adversarial attacks remains a fundamental security concern. Adversarial attacks generate negligible and imperceptible perturbations to DNN inputs, leading to severe misclassifications and raising risks in safety-critical and security-critical applications~\citep{szegedyIntriguingPropertiesNeural2014,goodfellowExplainingHarnessingAdversarial2014,eykholtRobustPhysicalWorldAttacks2018,biggioWildPatternsTen2018}. A key property of these attacks is that they exploit the gradient-based optimization used to train models in supervised learning (SL)~\citep{goodfellowExplainingHarnessingAdversarial2014}. In standard SL settings, loss gradients are often smooth, stable, and highly informative, enabling efficient adversarial optimization via methods such as projected gradient descent (PGD)~\citep{madryDeepLearningModels2018}. As a result, the gradient itself becomes an attack surface: once reliable gradients are available, attackers can efficiently construct adversarial examples that manipulate model predictions while remaining imperceptible~\citep{goodfellowExplainingHarnessingAdversarial2014, eykholtRobustPhysicalWorldAttacks2018}.

To mitigate this vulnerability and increase model robustness, a wide range of defenses have been proposed, including noise injection, data augmentation, and adversarial training~\citep{bishopTrainingNoiseEquivalent1995,cohenCertifiedAdversarialRobustness2019,hendrycksAugMixSimpleData2020,zhangMixupEmpiricalRisk2018,madryDeepLearningModels2018,kurakinAdversarialMachineLearning2017}. However, prior work has shown that even robust models can be compromised by adaptive attackers, especially when defenses rely on gradients that remain exploitable on the training set. These observations raise a fundamental question: \textbf{\textit{Can alternative training paradigms alter gradient structure in a way that fundamentally decreases adversarial attack effectiveness?}}

Reinforcement learning (RL), another core paradigm in machine learning, provides a different optimization mechanism based on exploration and policy gradients, rather than purely deterministic gradient descent. While RL has been widely studied in control and sequential decision-making tasks~\citep{mnihHumanlevelControlDeep2015,levineEndtoEndTrainingDeep2016,pintoRobustAdversarialReinforcement2017,akhtarThreatAdversarialAttacks2018,biggioWildPatternsTen2018}, its role in adversarial robustness for classification remains underexplored.

In this work, we investigate how RL influences adversarial robustness by comparing RL-trained classifiers with their SL counterparts. Our central hypothesis is that RL alters gradient structure in a way that impacts the feasibility of gradient-based adversarial optimization. We conduct experiments on CIFAR-10, CIFAR-100, and ImageNet-100, and observe that RL-trained models exhibit strong robustness to gradient-based attacks such as PGD. For instance, adversarial accuracy of RL-trained models under PGD increases substantially compared to SL-trained models, while clean accuracy decreases only slightly. However, this robustness does not generalize to transfer-based attacks. Adversarial examples generated on RL models are difficult to obtain, whereas adversarial examples generated on supervised models transfer effectively and significantly degrade the performance of RL-trained models. We further show that this asymmetric robustness arises from a fundamental difference in gradient behavior between SL and RL. Compared to SL, RL training induces smaller gradient magnitudes and highly unstable gradient directions, which disrupt the optimization process used by gradient-based attackers. As a result, adversarial directions are difficult to identify via direct optimization, while adversarial regions in the input space persist in RL-trained models. Consequently, adversarial examples generated on other models can still transfer successfully. This finding highlights both the strength and the limitations of RL-based training. On the one hand, RL increases robustness against gradient-based attacks by degrading gradient usability through smaller magnitudes and unstable directions. On the other hand, it does not eliminate adversarial regions in the input space, leaving models vulnerable to transfer-based attacks. This suggests that robustness gained from disrupting optimization during adversarial attack is insufficient for comprehensive security. \textbf{\textit{Instead, effective defenses should combine RL-induced gradient instability with adversarial training in order to improve robustness against both gradient-based and transfer-based adversaries.}} Our main contributions are summarized as follows:
\begin{enumerate}
    \item We demonstrate that RL-based training significantly reduces the effectiveness of gradient-based adversarial attacks.
    \item We identify a mechanism based on reduced gradient magnitude and increased gradient instability that disrupts adversarial optimization dynamics.
    \item We uncover a critical limitation: adversarial examples generated from supervised models transfer effectively to RL-trained models, showing that degrading gradients does not eliminate adversarial vulnerability.
    \item We show that combining RL with adversarial training (RL-adv) improves robustness across both gradient-based and transfer-based threat models.
\end{enumerate}

\section{Related Work}\label{sec:relatedwork}
Deep neural networks (DNNs) trained with supervised learning (SL) have achieved remarkable success in image classification, surpassing human performance on several vision tasks~\citep{heDelvingDeepRectifiers2015}. However, in addition to the capability of DNN-based image classifiers, they also show susceptibility to a wide range of attacks~\citep{ozdagAdversarialAttacksDefenses2018}. DNNs are susceptible to privacy and security threats such as (i) data poisoning~\citep{zhaoDataPoisoningDeep2025}, where adversaries inject poisoned data into the training set to corrupt the model, (ii) evasion (also known as adversarial attacks), where input samples are intentionally perturbed in a way that causes the model to misclassify them during the testing phase, (iii) model inversion, where the goal is feature reconstruction of samples from the training set using the model’s outputs~\citep{fredriksonModelInversionAttacks2015}, (iv) membership inference attacks, which attempt to determine whether a specific data sample was part of the model’s training dataset or not, etc~\citep{shokriMembershipInferenceAttacks2017}. From all these attacks, evasion attacks are the most well-known ones against DNNs~\citep{ilyasAdversarialExamplesAre2019}. 

These types of attacks can be considered especially practical because they exploit non-robust features learned by DNNs, making them particularly concerning because even naturally occurring perturbations can sometimes induce misclassification behavior similar to adversarial examples \citep{ilyasAdversarialExamplesAre2019}. The perturbations used to generate adversarial examples may be perceptible \citep{schneiderDualAdversarialAttacks2023}, where the attacker aims to deceive both humans and DNNs, or imperceptible to the human eye \citep{aghabagherlooImpactDataDuplication2025, ilyasAdversarialExamplesAre2019}. In most scenarios, adversaries prefer imperceptible perturbations because they can mislead the DNN while remaining visually indistinguishable to humans.
 Adversarial attacks are generally categorized into white-box and black-box settings. In white-box attacks, the adversary has full access to the model architecture, parameters, and gradients. Most white-box attacks are gradient-based, meaning that they use the gradient of the loss with respect to the input to determine how to perturb the sample in order to maximize the prediction error. Representative gradient-based attacks include Projected Gradient Descent (PGD), Fast Gradient Sign Method (FGSM), and Carlini \& Wagner (C\&W) attacks.

In black-box attacks, the adversary has limited or no knowledge of the model parameters and can only interact with the model. Black-box attacks are commonly divided into transfer-based and query-based attacks. Transfer-based attacks generate adversarial examples on a substitute model and then apply them to the target model, relying on the transferability property of adversarial examples. Query-based attacks directly estimate gradients or decision boundaries by repeatedly querying the target model. 

%\todo{rewrite, e.g., rain in highway camera to mislead...}

\subsection{Gradient-based Attacks on Classification Task}\label{sec:aeonclassification}

Gradient-based attacks are the most widely studied adversarial attacks in classification tasks. These attacks exploit the gradient of the model loss with respect to the input image to identify the direction in which the input should be perturbed to maximize the loss function. Since they rely on access to model gradients, they are typically studied in white-box settings.

One of the earliest gradient-based methods is the Fast Gradient Sign Method (FGSM), which generates adversarial examples in a single step by perturbing the input in the direction of the sign of the gradient \citep{goodfellowExplainingHarnessingAdversarial2014}. Iterative variants, such as the Basic Iterative Method (BIM) and Projected Gradient Descent (PGD), repeatedly apply small perturbations and project the perturbed sample back into a valid perturbation region, making them more effective than single-step attacks \citep{madryDeepLearningModels2018}. Another widely used attack is the Carlini and Wagner (C\&W) attack, which formulates adversarial example generation as an optimization problem and can produce highly effective and often imperceptible perturbations \citep{carliniEvaluatingRobustnessNeural2017}.

\subsection{Transfer-based and Query-based Adversarial Attacks on Classification Task}
As explained, transfer-based attacks are black-box attacks in which adversarial examples are generated on a surrogate model and then transferred to a target model. Their success relies on the transferability property of adversarial examples, meaning that perturbations crafted for one model can often fool other models as well \citep{aghabagherlooUnveilingIllusionaryRobust2025}. These attacks do not require access to the target model’s gradients; instead, attackers can train a substitute model and use white-box attacks such as FGSM or PGD to craft transferable adversarial examples.

%Several methods have been proposed to improve transferability. Momentum-based attacks stabilize the perturbation direction across iterations and improve transfer success \citep{dongBoostingAdversarialAttacks2018}. Ensemble-based attacks generate adversarial examples using multiple surrogate models to increase their generalization to unseen target models \citep{liuDelvingTransferableAdversarial2017}.

Query-based attacks are another type of black-box attack in which the adversary interacts directly with the target model through repeated queries. These attacks estimate gradients or decision boundaries using only the model outputs. ZOO is one of the most well-known query-based attacks and estimates gradients using finite differences \citep{chenZooZerothOrder2017}.  The other well-known attack is Square Attack, which aims to reduce the number of required queries \citep{andriushchenko2020square}.

%Other methods include NES, which uses random search to estimate gradients more efficiently \citep{ilyasBlackboxAdversarialAttacks2018}, Boundary Attack, which iteratively reduces adversarial perturbations while preserving misclassification \citep{brendelDecisionBasedAdversarialAttacks2018}, and Square Attack, which aims to reduce the number of required queries \citep{andriushchenko2020square}.

To defend against these attacks, several studies have proposed adversarial training and other robustification techniques \citep{wuDefensesAdversarialMachine2023}. However, subsequent work has demonstrated that many defenses remain vulnerable when the attacker knows the defense mechanism itself \citep{aghabagherlooBrittlenessRobustFeatures2023, athalyeObfuscatedGradientsGive2018}. This has led to a recurring cycle in the literature, where new defenses are proposed and later bypassed by stronger attacks.

\subsection{Robustness and Reinforcement Learning}
Although RL’s contribution to robustness has been rarely explored in classification, prior studies report that RL can yield robust behavior in control and robotics, supported by worst-case optimization viewpoints that treat environment uncertainty explicitly during learning~\citep{pintoRobustAdversarialReinforcement2017,rajeswaranEPOptLearningRobust2017, nilimRobustControlMarkov2005,wiesemannRobustMarkovDecision2013,dermanDistributionalRobustnessRegularization2020}. 

% In control and robotics, training an agent together with an explicit disturbance opponent, or repeatedly training on the hardest trajectories drawn from a diverse set of environment settings, improves stability and transfer to shifted dynamics~\cite{pinto2017rarl,rajeswaran2017epopt}; these practices are supported by worst-case optimization viewpoints that treat environment uncertainty explicitly during learning~\cite{nilim2005robust,wiesemann2013robust,derman2020dro}. 

For input perturbations more similar to evasion attacks, a lightweight sensitivity penalty discourages large output changes when the input is changed slightly. Prior work demonstrates consistent robustness gains across common algorithms with minimal loss in clean performance, while placing stronger emphasis on early decisions further mitigates behavior drift over time~\citep{zhangRobustDeepReinforcement2020,yamabeRobustDeepReinforcement2024}. Beyond pixel changes, multi-agent studies demonstrate that an opponent can steer a victim policy into harmful behavior without modifying pixels directly, underscoring the need to test opponent-driven threats as well~\citep{gleaveAdversarialPoliciesAttacking2020}. For safety-critical cases, online selection rules that prefer actions remaining good under bounded input noise have been shown to improve resilience and come with simple certificates~\citep{lutjensCertifiedAdversarialRobustness2020}. Evidence outside control also points in the same direction: an RL-style generator–classifier training improves robustness to lexical substitutions in text classification~\citep{xuLexicalATLexicalBasedAdversarial2019}, and RL-based sequential feature acquisition improves resilience when the model must decide which features to read before predicting~\citep{janischClassificationCostlyFeatures2020}. However, despite these advances in adjacent areas, a systematic comparison where RL serves as the \textbf{primary} training paradigm for adversarially robust \textbf{image classification}, under standard white-box and black-box attacks and matched training budgets, remains limited.

\section{Threat Model}\label{sec:threat_model}
We define our threat model following the conventions of adversarial machine learning in security-critical settings. We specify the target system, adversary goals, adversary capabilities under three distinct attack models, and the scope of our analysis.

\subsection{Target System}
We consider a neural network image classifier deployed for inference. The model takes an input image $X$ and outputs a predicted class label $\hat{y} = f_\theta(X)$. We study four model variants that share the same architecture but differ in their training paradigm: 
\begin{itemize}
    \item \textbf{SL}: trained via supervised learning (cross-entropy loss);
    \item \textbf{SL-adv}: trained via supervised learning with adversarial training (TRADES);
    \item \textbf{RL}: trained via reinforcement learning (policy gradient with $\varepsilon$-greedy exploration);
    \item \textbf{RL-adv}: trained via reinforcement learning with adversarial training (TRADES).
\end{itemize}
No additional defense mechanisms (e.g., input preprocessing, randomized smoothing, or ensemble methods) are applied beyond the training procedure itself. This allows us to isolate the effect of the training paradigm on the adversarial attack surface.

\subsection{Adversary Goal}
The adversary aims to cause misclassification by crafting an adversarial example $X'$ from a clean input $X$ such that the model predicts an incorrect label, i.e., $f_\theta(X') \neq y$, while keeping the perturbation imperceptible. We focus on untargeted attacks under $\ell_p$-norm-bounded constraints:
\begin{equation}
    \|X' - X\|_p \leq \varepsilon_{\text{adv}},
\end{equation}
where $p \in \{2, \infty\}$ and $\varepsilon_{\text{adv}}$ is the perturbation budget specified per dataset (see Section~\ref{sec:implementation} for exact values).

\subsection{Adversary Capabilities}

We evaluate robustness under three adversary types with progressively restricted access to the target model. For each type, we specify what the adversary knows and what the adversary can do.

\paragraph{Gradient-based adversary (white-box).}
The adversary has full knowledge of the target model, including the architecture, trained parameters $\theta$, and the training paradigm (SL or RL). The adversary can compute exact gradients $\nabla_x \mathcal{L}(f_\theta(x), y)$ with respect to the input and perform iterative optimization to craft adversarial examples. This is the strongest attack setting we consider. Representative attacks include PGD~\cite{madryDeepLearningModels2018} and Auto-PGD (APGD)~\cite{croce2020reliable}. This adversary type directly tests our central hypothesis: whether RL-induced gradient properties degrade the effectiveness of gradient-based optimization.

\paragraph{Transfer-based adversary (black-box, surrogate access).}
The adversary does not have access to the target model's parameters or gradients, but can train a surrogate model on the same dataset. In our evaluation, the surrogate model shares the same architecture as the target but is trained under a different paradigm (e.g., SL surrogate attacking an RL target, or vice versa). The adversary generates adversarial examples on the surrogate using white-box attacks and applies them to the target, exploiting the transferability of adversarial examples across models~\cite{aghabagherlooImpactDataDuplication2025}. This setting is particularly relevant to our analysis: if an adversary knows the target is RL-trained and thus resistant to direct gradient attacks, a natural adaptive strategy is to craft adversarial examples on an SL-trained surrogate instead.

\paragraph{Query-based adversary (black-box, query access only).}
The adversary can submit inputs to the target model and observe the output predictions (class labels or confidence scores), but has no access to model parameters, gradients, or training details. The adversary estimates adversarial directions through repeated queries. We adopt the Square Attack~\cite{andriushchenko2020square}, a query-efficient method that does not rely on gradient estimation but instead performs randomized search within the $\ell_\infty$ perturbation budget. This adversary type tests whether the robustness observed under gradient-based attacks extends to settings where the attacker bypasses gradients entirely.

\subsection{Scope}

Our primary goal is to understand how training paradigms shape the adversarial attack surface and to identify the mechanism by which RL-based training disrupts gradient-based adversarial optimization. Based on this understanding, we further demonstrate that combining RL with adversarial training (RL-adv) provides a dual-layer defense that outperforms standard adversarial training alone. 

% \subsection{Key Observation}

% Our results reveal that RL-based training significantly reduces the effectiveness of gradient-based attacks but does not prevent transfer-based attacks. This suggests that disrupting optimization does not eliminate adversarial regions in the input space.

\section{Methodology}\label{sec:methodology}
%\ali{I just have a general comment on the paper, but this is just my suggestion. I think it would be better if you structure the paper like a story, where each part contributes to the overall narrative. Something like: 1-First, present the problem statement (and maybe include some numbers showing that previous robustification methods exhibit susceptibility, adding results that demonstrate this vulnerability). 2-Explain why you think employing RL is helpful. 3-Present the experimental results — a really detailed explanation is necessary here, for example: RL in this scenario with epsilon equal to 0.5 showed this, but with epsilon equal to 1 showed something different, and explain why. 4-Add a theoretical framework that covers and supports the experimental results.}

%\todo{(RL classification image, we are the first person to use the reinforcement learning-based method for image classification with adversarial training. This should be clear.)} 
In this section, we present the training paradigms, adversarial training procedure, attack methods, and gradient analysis framework used in our study. We introduce two training paradigms, i.e., supervised learning (SL) and reinforcement learning (RL), and formulate classification under each. We then describe how adversarial training is incorporated into both paradigms under identical configurations to ensure a fair comparison. Finally, we define the adversarial attacks used for evaluation and the analytical tools used to characterize gradient-level differences between SL and RL.

\subsection{Training Paradigms}\label{sec:training_paradigms}
\subsubsection{Supervised Learning}\label{sec:sl}
In supervised learning, the model $f_\theta(X)$ maps an input image $X$ to a probability distribution over $C$ predefined classes~\citep{krizhevskyImagenetClassificationDeep2012,heDeepResidualLearning2016,dosovitskiyImageWorth16x162020}. The parameters $\theta$ are optimized by minimizing the cross-entropy loss:
\begin{equation}
    \mathcal{L}_{CE} = -\sum_{c=1}^{C} \mathbb{I}(y = c) \cdot \log \hat{p}_c,
\end{equation}
where $y$ is the ground-truth label, $\hat{p} = f_\theta(X)$ is the predicted probability vector, and $\mathbb{I}(\cdot)$ is the indicator function. The overall objective is:
\begin{equation}
    \theta^* = \arg\min_\theta \mathbb{E}_{(X,y) \sim \mathcal{D}} [\mathcal{L}_{CE}(f_\theta(X), y)].
\end{equation}
Optimization is performed via backpropagation and gradient descent. A key property of this formulation is that each training sample $(X, y)$ provides a deterministic gradient signal: given a fixed input and label, the loss and its gradient corresponding to the input are fully determined.

\subsubsection{Reinforcement Learning}\label{sec:rl}
We formulate classification as a one-step policy learning problem. The model defines a policy $\pi_\theta(a \mid X)$ over class labels, where the state is the input image $X$ and the action $a$ is the predicted class. For each input, an action $a_t$ is sampled from the policy, and a reward $r_t$ is assigned as $r_t = \mathbb{I}(a_t = y)$. The policy is optimized using a REINFORCE-style objective~\cite{williams1992simple}:
\begin{equation}
    \mathcal{L}_{PG} = -\mathbb{E}_{a \sim \pi_\theta} [\log \pi_\theta(a_t \mid X) \cdot \mathbb{I}(a_t = y)].
\end{equation}
This objective increases the likelihood of actions that yield positive rewards. To encourage sufficient exploration, we adopt an $\varepsilon$-greedy strategy: with probability $\varepsilon$, the action is sampled uniformly at random; otherwise, it is sampled from the policy distribution $\pi_\theta(a \mid X)$.

\subsubsection{Structural difference in optimization signals.}\label{sec:sl_rl_difference}
A key distinction between SL and RL shows in the nature of the gradient signal used for parameter updates. In SL, the cross-entropy loss provides a \textbf{deterministic} gradient for each sample: given a fixed $(X, y)$, the gradient $\nabla_\theta \mathcal{L}_{CE}$ is uniquely determined. In RL, the gradient signal is inherently \textbf{stochastic} for two reasons: (1)~the action $a_t$ is sampled from $\pi_\theta$ rather than deterministically assigned, meaning that different samples of $a_t$ yield different gradient directions even for the same input; and (2)~the $\varepsilon$-greedy exploration mechanism further randomizes action selection, introducing additional variance into the gradient estimate. As a result, RL parameter updates are driven by a noisier, higher-variance and less smooth gradient signal compared to SL. Whether such differences in optimization gradient affect adversarial robustness is the central empirical question investigated in Sections~\ref{sec:results}, ~\ref{sec:analysis} and~\ref{sec:discussion}.

\subsection{Adversarial Training}\label{sec:adv_training}
To study robustness under adversarial settings, we incorporate adversarial training into both SL and RL. Importantly, both SL-adv and RL-adv use an identical adversarial training procedure; the only difference between them is the base training paradigm (cross-entropy vs.\ policy gradient). This design ensures that any robustness difference between SL-adv and RL-adv can be attributed to the training paradigm itself, rather than to differences in the adversarial training configuration. Adversarial examples are generated using FGSM and incorporated into training as data augmentation:
\begin{equation}
    X_{\text{adv}} = \text{clip}(X + \epsilon \cdot \text{sign}(\nabla_X \mathcal{L}),\; 0,\; 1),
\end{equation}
where $\epsilon$ denotes the perturbation budget.
 
Additionally, we integrate adversarial examples into the TRADES framework~\cite{zhangTheoreticallyPrincipledTradeoff2019}, which explicitly balances standard accuracy and robustness:
\begin{equation}
\label{eq:trades}
\begin{aligned}
\theta^{*} = \arg\min_{\theta} \mathbb{E}_{(X,y)} \Big\{ 
& \underbrace{\mathcal{L}(f_\theta(X), y)}_{\text{Standard Accuracy}} \\
&+ \beta \cdot \underbrace{
\max_{I_{adv} \in \mathbb{B}(X,\epsilon)}
\mathcal{L}(f_\theta(X), f_\theta(X_{adv}))
}_{\text{Robustness Regularizer}}
\Big\}
\end{aligned}
\end{equation}
where $\beta$ controls the trade-off between accuracy and robustness. In our setting, $\beta$ is scheduled during training: smaller values emphasize accuracy in early stages, while larger values improve robustness later (see Section~\ref{sec:implementation} for exact scheduling).

\subsection{Adversarial Attacks}\label{sec:attacks}
We evaluate robustness under three attack types, corresponding to the adversary types defined in our threat model (Section~\ref{sec:threat_model}).
 
\subsubsection{Gradient-based attack.}
PGD~\cite{madryDeepLearningModels2018} is a strong gradient-based attack that iteratively updates adversarial examples within an $\ell_p$-bounded constraint:
\begin{equation}
X^{t+1} = \Pi_{X+\mathcal{S}} \left( X^t + \alpha \cdot \phi(\nabla_X J(\theta, X^t, y)) \right)
\end{equation}
where $X^t$ is the adversarial example at iteration $t$, $\alpha$ is the step size, $\Pi_{X+\mathcal{S}}$ denotes projection onto the perturbation set $\mathcal{S} = \{\delta \mid \|\delta\|_p \leq \epsilon\}$, and $\phi(\cdot)$ denotes the norm-dependent step direction: $\phi(g) = \mathrm{sign}(g)$ for $\ell_\infty$ and $\phi(g) = g / \|g\|_2$ for $\ell_2$. We additionally use AutoAttack~\cite{croce2020reliable}, a standardized ensemble of complementary gradient-based attacks (APGD-CE, APGD-T), as a stronger evaluation baseline.
 
\subsubsection{Transfer-based attack.}
Adversarial examples are generated on a surrogate model using white-box attacks (PGD) and then applied to the target model. In our setup, surrogate and target models share the same architecture but differ in training paradigm, allowing us to evaluate whether adversarial examples generated under one paradigm transfer effectively to models trained under a different paradigm.
 
\subsubsection{Square-based attack.}
Square Attack~\citep{andriushchenko2020square} is a query-efficient black-box attack that perturbs the input by iteratively modifying randomly selected square regions within the $\ell_\infty$ budget. It does not rely on gradient information, providing a complementary evaluation of robustness in settings where the attacker bypasses gradients entirely.

\subsection{Gradient Analysis Framework}\label{sec:gradient_analysis_framework}
To investigate \textbf{why} different training paradigms lead to different robustness results, we introduce a set of analytical tools that characterize the gradient structure of trained models. These tools are applied in Sections~\ref{sec:analysis} and~\ref{sec:discussion} to support the interpretation of our empirical results. We organize them into three categories: loss geometry visualization, static gradient indicators, and dynamic gradient indicators. Formal definitions are provided in Appendix~\ref{app:robustness_indicators}.
 
\subsubsection{Loss geometry visualization.}
Decision-boundary diagrams visualize classification regions under adversarial perturbations, revealing the sharpness or flatness of boundaries~\citep{fawziEmpiricalStudyTopology2018,moosavi-dezfooliDeepfoolSimpleAccurate2016}. Loss-landscape diagrams plot the scalar loss along the adversarial attack direction and a perpendicular direction, making the gradient magnitude visually interpretable~\citep{liVisualizingLossLandscape2018,liuLossLandscapeAdversarial2020}.
 
\subsubsection{Static gradient indicators.}
These indicators characterize the gradient field of a trained model \emph{before any attack is applied}:
\begin{itemize}
    \item \textbf{Average Gradient Norm (AGN)}: measures the expected magnitude of input gradients $\|\nabla_x \mathcal{L}(x, y)\|_2$ over the data distribution~\citep{moosavi-dezfooliDeepfoolSimpleAccurate2016}. Larger AGN indicates greater sensitivity to input perturbations.
    \item \textbf{Input-Gradient Variance (IGV)}: captures the variance of the input gradient under small Gaussian perturbations of the input~\citep{wangEnhancingTransferabilityAdversarial2021,agarwalEstimatingExampleDifficulty2022}. Higher IGV indicates larger and more variable gradient magnitudes.
    \item \textbf{Directional Input-Gradient Variance (dIGV)}: measures the variability of the \emph{direction} (not magnitude) of the input gradient under perturbations~\cite{liuEnhancingGeneralizationUniversal2023,dengEnhancingTransferabilityTargeted2023}. Higher dIGV indicates less consistent gradient directions, making it harder for an attacker to identify a stable adversarial direction.
\end{itemize}
 
\subsubsection{Dynamic gradient indicators.}
These indicators track how gradients evolve \emph{during} iterative adversarial optimization:
\begin{itemize}
    \item \textbf{Gradient Stability Under Attack (GSUA)}: records the cosine similarity between consecutive attack gradients across PGD iterations. High GSUA indicates a coherent ascent direction favorable to the attacker; low or negative GSUA signals rapidly changing directions that hinder attack convergence.
    \item \textbf{Gradient norm trajectory}: tracks the $\ell_2$ norm of the attack gradient across iterations, separating directional instability from changes in gradient scale.
\end{itemize}
 
\subsubsection{Predictive entropy.}
Mean predictive entropy $H$ summarizes the dispersion of the model's output distribution under clean and perturbed inputs~\cite{smithUnderstandingMeasuresUncertainty2018,kopetzkiEvaluatingRobustnessPredictive2021,qinImprovingCalibrationRelationship2021,emdeCertificationUncertaintyCalibration2024}. Higher entropy under attack indicates that the model retains uncertainty rather than producing overconfident (but incorrect) predictions.

\section{Implementation}\label{sec:implementation}

This section describes the experimental implementation used in our evaluation. We present the model architectures, datasets, training settings, attack configurations, and evaluation metrics used to compare SL, RL, SL-adv, and RL-adv under a consistent setup.

\subsection{Models}\label{sec:models}
To evaluate robustness across architectures of varying complexity, we use three neural network models: a 4-layer CNN, a 6-layer CNN, and ResNet-18~\cite{heDeepResidualLearning2016}. The architectures of the 4-layer and 6-layer CNNs are illustrated in Figure~\ref{fig:4_and_6_layer_CNN}. 

\begin{figure}[htbp]
    \centering
    \includegraphics[width=1.0\linewidth]{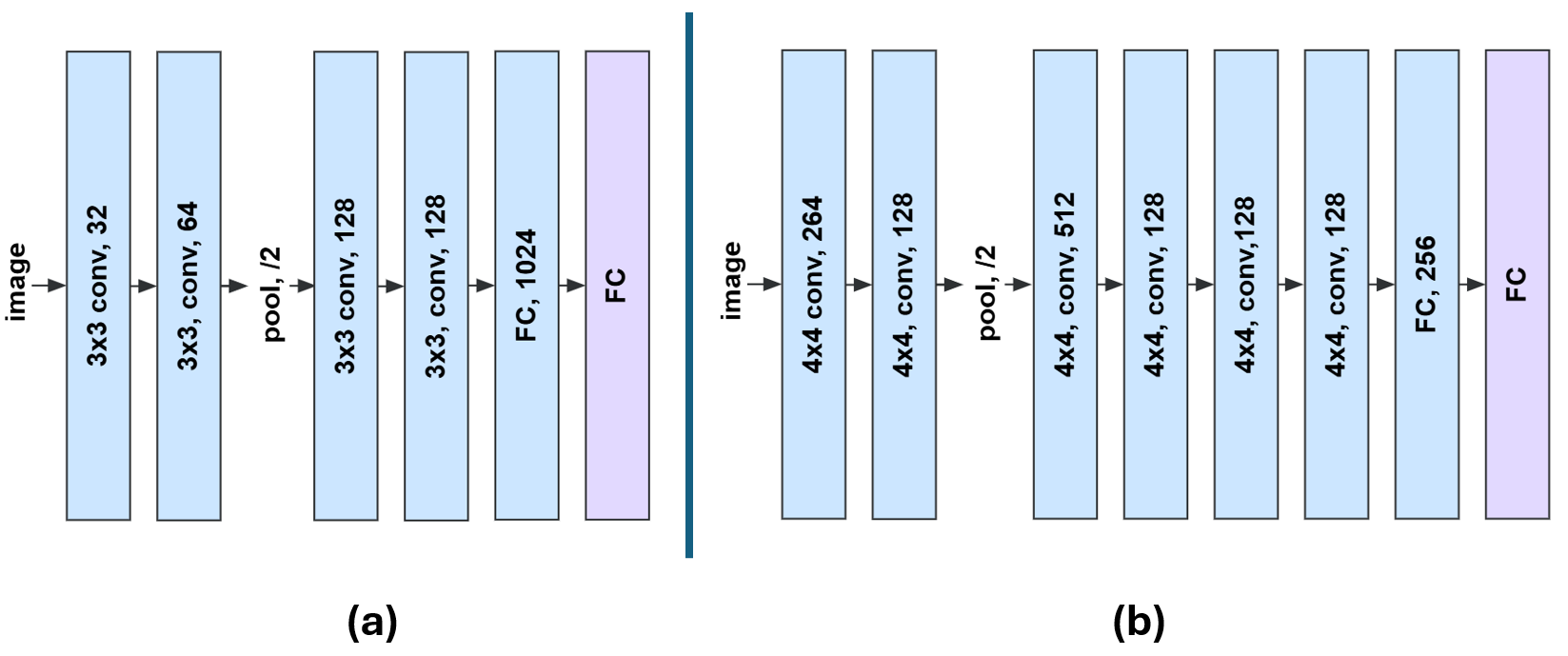}
    \caption{Model architectures of the (a) 4-layer CNNs and (b) 6-layer CNNs used for evaluation.}
    \label{fig:4_and_6_layer_CNN}
\end{figure}

\subsection{Datasets}\label{sec:datasets}
We conduct experiments on three benchmark datasets: CIFAR-10~\cite{krizhevskyLearningMultipleLayers2009}, CIFAR-100~\cite{krizhevskyLearningMultipleLayers2009}, and ImageNet-100~\cite{tianContrastiveMultiviewCoding2020}, as shown in Table~\ref{tab:dataset_stats}.

\begin{table}[htbp]
\caption{Benchmark dataset specifications.}
\label{tab:dataset_stats}
\begin{center}
\begin{tabular}{llll}
\toprule
\multicolumn{1}{l}{\bf Dataset} & 
\multicolumn{1}{l}{\bf Train/Test samples} & 
\multicolumn{1}{l}{\bf Image Size} & 
\multicolumn{1}{l}{\bf Classes} \\ 
\hline 
CIFAR-10    & 50,000/10,000 & 32$\times$32$\times$3 & 10 \\
CIFAR-100   & 50,000/10,000 & 32$\times$32$\times$3 & 100 \\
ImageNet-100 & 126,689/5,000 & 224$\times$224$\times$3 & 100 \\
\bottomrule
\end{tabular}
\end{center}
\end{table}

\subsection{Training Configuration}\label{sec:trainconf}
All models are trained under both SL and RL settings with consistent architectures and batch sizes to ensure a fair comparison.
 
\textbf{Supervised Learning (SL).}
We use the Adam optimizer with an initial learning rate of 0.005, decayed to 0.0005 during the final 20 epochs. Each model is trained for 100 epochs with a batch size of 256.
 
\textbf{Reinforcement Learning (RL).}
We use the Adam optimizer with a learning rate of 0.005. Due to the higher variance in policy optimization, models are trained for 2000 epochs to ensure convergence. The batch size is set to 256, and an $\varepsilon$-greedy exploration strategy is adopted with $\varepsilon = 0.5$.
 
\textbf{Adversarial Training (SL-adv/RL-adv).}
Perturbations are generated using FGSM with a budget of $\epsilon = 3/255$. We adopt the TRADES framework with a scheduled regularization parameter: $\beta = 1$ in early training (prioritizing standard accuracy), gradually increasing to $\beta = 6$ (enhancing robustness).
 
% \paragraph{Computational cost.}
% RL training requires approximately $20\times$ more epochs than SL (2000 vs.\ 100) to achieve convergence. On a single NVIDIA A100 GPU, training a 6-layer CNN on CIFAR-10 takes approximately 20 minutes for SL and 200 minutes hours for RL. This computational overhead is a practical limitation discussed in Section~\ref{sec:limitation_future_work}.
 
% \paragraph{Reproducibility.}
% To account for optimization stochasticity, we train models with three independent random seeds $\{10, 20, 42\}$ and report mean results. The seed for adversarial attack evaluation is fixed to 33 across all experiments to ensure consistent comparison.

\subsection{Attack Configuration}\label{sec:attackconf}
In this study, we primarily focus on gradient-based attacks, while also considering transfer-based and query-based attacks to provide a comprehensive evaluation of robustness under different threat settings.

\textbf{Gradient-based attacks.}
We use Projected Gradient Descent (PGD) as the primary attack, implemented using the CleverHans library~\citep{papernotTechnicalReportCleverHans2018}. The attack is configured with a perturbation budget of \(\epsilon = 7.0\), 250 iterations, and a step size of \(\epsilon_{\text{iter}} = 0.3\). All input images are normalized to the $[0, 1]$ range, and the attack is conducted under an \(\ell_{2}\) norm constraint with a batch size of 256. We adopt a large perturbation budget to test model robustness under a strong attack regime; results across smaller budgets ($\epsilon \in \{0.25, 0.5, 1.0, 2.0, 3.5, 5.0\}$) are provided in Appendix~\ref{app:perturbation_budget} and confirm that the RL--SL gap holds consistently. Additionally, we include AutoAttack~\citep{croce2020reliable} as a stronger evaluation baseline. AutoAttack consists of multiple complementary attacks, including APGD-CE, APGD-T~\citep{croce2020reliable}, and FAB-T~\citep{croce2020minimally}, providing a more reliable assessment of model robustness.

\textbf{Transfer-based attacks.}
For transfer attacks, adversarial examples are generated from models trained under different settings (SL, RL, SL-adv, and RL-adv) using the same dataset and model architecture. These adversarial examples are then applied to other models to evaluate cross-model transferability. This setup allows us to examine whether adversarial examples generated under one training paradigm remain effective against models trained under different paradigms.

\textbf{Query-based attacks.}
We adopt the SQUARE attack~\citep{andriushchenko2020square}, a query-efficient black-box attack, to evaluate robustness in settings where gradients are not accessible. This allows us to further analyze model behavior under strict black-box attack scenarios.

\subsection{Evaluation Metrics}

We use classification accuracy under adversarial attacks as the primary metric to evaluate model robustness. Higher accuracy indicates stronger robustness against adversarial perturbations. The accuracy is computed as:
\begin{equation}
\text{Accuracy} = \frac{1}{N} \sum_{i=1}^{N} \mathbb{I}\big(f_{\theta}(X_i) = y_i\big),
\end{equation}
where \(N\) is the number of samples, \(X_i\) and \(y_i\) denote the input and its ground-truth label, respectively, and \(\mathbb{I}(\cdot)\) is the indicator function that equals 1 if the prediction is correct and 0 otherwise. For adversarial evaluation, \(X_i\) is replaced by the adversarial example \(X_i^{adv}\).

% \ali{I still think 6 pages for intro+ primarily concepts is a bit a lot, but if you believe it is necessary, keep it as it is }

%as measured by the Attack Success Rate (ASR):
%\begin{equation}
%\text{ASR} = \frac{1}{N}\sum_{i=1}^N \mathbb{I}(f_\theta(\mathbf{X}_i') \neq y_i)
%\end{equation}

%where \(\mathbf{X}_{i}'\) denotes the adversarial example crafted from input \(\mathbf{X}_{i}\), and \(\mathbb{I}\) is the indicator function. The \(\mathbf{\ell}_2\)-norm bound \(\epsilon\) ensures perturbation imperceptibility.

\section{Robustness Evaluation}\label{sec:results}
We evaluate the robustness of SL- and RL-trained models across three attack settings: gradient-based, transfer-based, and query-based. All four model variants (SL, SL-adv, RL, RL-adv) are evaluated on CIFAR-10, CIFAR-100, and ImageNet-100. The main phenomena are observed consistently across all three backbone architectures (4-layer CNN, 6-layer CNN, ResNet-18); we focus on the 6-layer CNN in the main text and report complete results in Appendix~\ref{app:model_performance}. Figure~\ref{fig:adv_examples} illustrates that the adversarial perturbations used in our evaluation remain imperceptible to human observers (e.g., CIFAR-10: true label ``truck'', predicted label before attack ``truck'', predicted label after attack ``ship'').

\begin{figure}[htbp]
    \centering
    \includegraphics[width=1.0\linewidth]{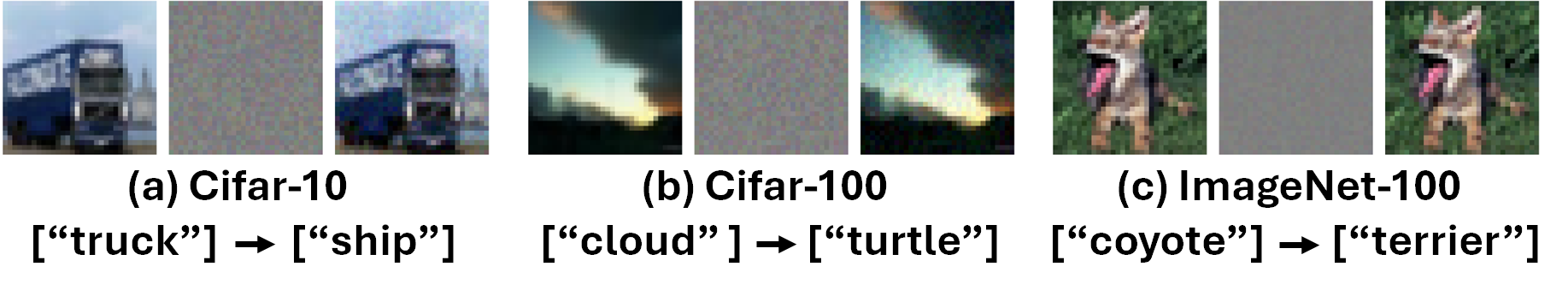}
    \caption{Visualization of adversarial examples attacked on 6-layer-CNN-SL across benchmark datasets: (a) CIFAR-10 example showing original image (left), additive perturbation (middle), and adversarial image (right); (b) CIFAR-100 example; (c) ImageNet-100 example.}
    \label{fig:adv_examples}
\end{figure}

\subsection{Gradient-Based Attacks}\label{sec:gradient_based_attack}

We first evaluate robustness under gradient-based adversaries, which construct adversarial examples by directly optimizing perturbations using model gradients. This corresponds to the primary attack mechanism targeted in our analysis.

\paragraph{PGD attack.}
Table~\ref{tab:model_performance_across_datasets} reports clean accuracy and adversarial accuracy under PGD for the 6-layer CNN across all three datasets. Two key observations emerge. (1) First, RL and RL-adv achieve clean accuracy only 3--5\% lower than SL and SL-adv, confirming that the RL formulation achieves competitive classification performance. (2) Second, under PGD attack, RL-trained models substantially outperform SL-trained models: adversarial accuracy increases from approximately 5\% to 56\% on CIFAR-10, from 3\% to 13\% on CIFAR-100, and from 6\% to 18\% on ImageNet-100. This suggests that gradient-based optimization struggles to find effective adversarial directions on RL-trained models.

\begin{table}[htbp]
\caption{Model robustness evaluation (\%) across datasets.}
\label{tab:model_performance_across_datasets}
\centering
\resizebox{\linewidth}{!}{
\begin{tabular}{lcccccc}
\toprule
 & \multicolumn{2}{c}{\bf CIFAR-10} & \multicolumn{2}{c}{\bf CIFAR-100} & \multicolumn{2}{c}{\bf ImageNet-100} \\
\cline{2-7}
\bf Model & \bf Clean & \bf AE & \bf Clean & \bf AE & \bf Clean & \bf AE \\ 
\hline
6-layer-CNN-SL & \bf 90.74* & 5.00 & \bf 64.75* & 2.53 & 57.64 & 5.72 \\
6-layer-CNN-SL-adv & 90.11 & 4.96 & 63.61 & 2.83 & \bf 58.00* & 5.36 \\
6-layer-CNN-RL & 88.50 & \bf 55.77* & 59.80 & 13.06 & 55.60 & 18.04 \\
6-layer-CNN-RL-adv & 87.63 & 48.63 & 56.54 & \bf 25.51* & 45.92 & \bf 18.24* \\
\hline
\end{tabular}
}
\end{table}

\paragraph{AutoAttack.}
Table~\ref{tab:aa_result_6layercnn_cifar10} shows AutoAttack results on CIFAR-10. Under this stronger evaluation, the advantage of plain RL over SL largely disappears: RL achieves 16.71\% under APGD-CE compared to 15.42\% for SL. This is expected because AutoAttack employs adaptive loss optimization that can discover adversarial directions even when gradients are unreliable. However, RL-adv consistently achieves the highest robustness across all AutoAttack components (36.27\% under APGD-CE vs.\ 24.87\% for SL-adv; 35.4\% under APGD-T and FAB-T vs.\ 21.77\% for SL-adv), suggesting that the combination of RL-induced gradient disruption and adversarial training provides complementary benefits that persist even under strong adaptive attacks.

\begin{table}[htbp]
\centering
\caption{Model robustness evaluation (\%) in CIFAR-10 dataset by AutoAttack.}
\label{tab:aa_result_6layercnn_cifar10}

\resizebox{\linewidth}{!}{
\begin{tabular}{lccccc}
\toprule
 & \textbf{Clean} & \textbf{APGD-CE} & \textbf{APGD-T} & \textbf{FAB-T} & \textbf{SQUARE} \\
\midrule
6-layer-CNN-SL      & \textbf{90.82*} & 15.42   & 13.58  & 13.58 & 11.6   \\
6-layer-CNN-SL-adv & 87.03 & 24.87   & 21.77  & 21.77 & 19.55   \\
6-layer-CNN-RL       & 88.62 & 16.71   & 13.73  & 13.73 & 11.96   \\
6-layer-CNN-RL-adv & 86.29 & \textbf{36.27*} & \textbf{35.4*} & \textbf{35.4*} & \textbf{32.41*} \\
\bottomrule
\end{tabular}
}
\end{table}

\subsection{Transfer-Based Attacks}\label{sec:transfer_based_attack}
We next evaluate robustness under transfer-based adversaries, where adversarial examples are generated on a surrogate model and then applied to a target model. Unlike gradient-based attacks, this setting does not rely on access to the target model’s gradients and instead exploits shared decision boundary structures across models. Table~\ref{tab:model_performance_across_models} reports model accuracy on adversarial examples generated from different source models. A clear asymmetry emerges: adversarial examples crafted on SL-based models transfer effectively to all target models, reducing accuracy to below 10\% in most cases. In contrast, adversarial examples generated from RL-based models are significantly less transferable, with all target models maintaining above 40\% accuracy. This asymmetry has two important implications. (1) First, it confirms that adversarial examples are substantially harder to generate via gradient-based optimization on RL models, which is consistent with the PGD results above. (2) Second, it reveals a critical limitation: despite resisting direct gradient attacks, RL-trained models remain vulnerable to adversarial examples transferred from SL surrogates. This indicates that RL training does not eliminate adversarial regions in the input space; it primarily disrupts the optimization process used to discover them. However, RL-adv addresses this limitation, achieving 30.88\% accuracy against SL-sourced adversarial examples (vs.\ 8.29\% for plain RL), demonstrating that combining adversarial training with RL-based optimization is necessary for complete robustness.

\begin{comment}
    Table~\ref{tab:model_performance_across_models} shows model accuracy on adversarial examples generated from different source models. We observe a clear asymmetry: adversarial examples crafted on SL-based models (SL and SL-adv) transfer effectively to all target models, reducing accuracy to below 10\% in most cases. In contrast, adversarial examples generated from RL-based models (RL and RL-adv) are significantly less transferable, with all models maintaining above 40\% accuracy. This asymmetry indicates that adversarial examples are substantially easier to generate on SL models, while RL models exhibit a hard-to-optimize property under direct gradient-based attacks. However, despite this difficulty, RL-trained models remain vulnerable to adversarial examples transferred from SL models. This suggests that RL training does not eliminate adversarial regions in the input space. Instead, it primarily disrupts the optimization process used to discover such regions. Once adversarial directions are identified on a surrogate model, they remain effective when transferred to RL-trained models. Plain RL provides resistance to adversarial examples generated from itself, but it does not prevent transfer-based attacks. Therefore, adversarial training is needed for RL. RL with adversarial training provides robustness against both strong (RL: 54.31\% / 46.91\%) and weak (SL: 30.88\% / 22.56\%) adversarial sources. This suggests that combining RL with adversarial training improves robustness beyond only optimization difficulty.

\end{comment}

\begin{table}[htbp]
\centering
\caption{Transfer attack robustness (\%) on CIFAR-10 using 6-layer CNN. Rows: target models. Columns: source models used to generate adversarial examples.}
\label{tab:model_performance_across_models}
\small
\renewcommand{\arraystretch}{0.9}
\resizebox{\linewidth}{!}{
\begin{tabular}{lcccc}
\toprule
\addlinespace[2pt]
& \multicolumn{4}{c}{\textbf{Adversarial Examples (AE)}} \\
\cline{2-5}
\textbf{Model} & \textbf{SL} & \textbf{SL-adv} & \textbf{RL} & \textbf{RL-adv} \\ 
\midrule
6-layer-CNN-SL      & 5.81 & 7.18 & 48.45 & 43.15  \\
6-layer-CNN-SL-adv  & 9.53 & 5.74 & 50.14 & 44.86  \\ 
6-layer-CNN-RL      & 8.29 & 7.92 & 48.84 & 42.69  \\
6-layer-CNN-RL-adv  & \textbf{30.88*} & \textbf{22.56*} & \textbf{54.31*} & \textbf{46.91*}  \\
\bottomrule
\end{tabular}
}
\end{table}

\subsection{Query-Based Attacks}\label{query_based_attacks}
We briefly evaluate robustness under query-based adversaries using the Square Attack, which does not rely on gradient information and instead approximates adversarial directions through model queries. Under Square Attack as shown in Table~\ref{tab:aa_result_6layercnn_cifar10}, RL-trained models show no significant advantage over SL-trained models (11.96\% vs.\ 11.6\%), confirming that RL's robustness gains are specific to gradient-dependent attacks. However, RL-adv again achieves the highest robustness (32.41\% vs.\ 19.55\% for SL-adv), suggesting that the flatter loss landscape induced by RL reduces the efficiency of query-based search even when gradient information is unavailable.

\begin{comment}
    Results in Table~\ref{tab:aa_result_6layercnn_cifar10} show that RL-trained models do not exhibit a significant advantage over SL-trained models under Square Attack. This indicates that the gains observed under gradient-based attacks do not directly extend to gradient-free settings, since RL disrupts optimization but does not eliminate adversarial regions. In contrast, RL-adv achieves higher robustness than SL-adv. This suggests that adversarial training improves decision boundary stability, while RL further lowers the efficiency of query-based search by making the loss landscape flatter and more irregular.
\end{comment}

% \subsection{Pareto Analysis: Clean–Robust Trade-off}
% To assess whether robustness improvements come at the expense of clean accuracy, we construct clean–robust Pareto frontiers for 6-layer CNNs by varying $\epsilon$ and $\beta_{\text{adv}}$, averaging over three random seeds with 95\% confidence intervals (details in Appendix~\ref{app:pareto}). Figures~\ref{fig:pareto_frontier} show that RL-adv consistently shifts the frontier upward relative to SL-adv under both PGD and AutoAttack, indicating a strictly more favorable trade-off. These results demonstrate that the robustness gains arise from improved optimization geometry rather than sacrificing clean accuracy.

% \begin{figure}[htbp]
%     \centering
%     \includegraphics[width=0.2\linewidth]{figs/pareto_pgd.png}
%     \caption{Clean–robust Pareto frontiers for 6-layer CNN under (a) PGD and (b) AutoAttack.}
%     \label{fig:pareto_frontier}
% \end{figure}

\subsection{Summary of Empirical Findings}\label{sec:results_summary}
 
Across all experiments, a consistent pattern emerges:
\begin{enumerate}
    \item \textbf{Plain RL} substantially resists gradient-based attacks (PGD) but provides little advantage under adaptive (AutoAttack) or gradient-free (Square) attacks, and remains vulnerable to transfer attacks from SL surrogates.
    \item \textbf{RL-adv} generally achieves the strongest robustness across attack types, outperforming SL-adv in the majority of settings.
    \item The robustness gap between RL and SL is largest for standard PGD and smallest for AutoAttack and Square, suggesting that RL's primary effect is disrupting gradient-based optimization rather than eliminating adversarial vulnerability.
\end{enumerate}
 
These observations raise a natural question: \emph{\textbf{What mechanism causes RL training to disrupt gradient-based adversarial optimization, and why is this effect insufficient on its own?}} We will address this in the following section.

% \section{Robustness Mechanism behind Reinforcement Learning}\label{sec:mechanism_rl}

\section{Mechanism Analysis}\label{sec:analysis}
The experimental results in Section~\ref{sec:results} show that RL training disrupts gradient-based attacks but does not eliminate adversarial vulnerability. In this section, we analyze the mechanism behind this phenomenon using the gradient analysis framework introduced in Section~\ref{sec:gradient_analysis_framework}. We examine three complementary perspectives: loss geometry, gradient instability analysis, and predictive uncertainty. 

\subsection{Decision boundary and loss geometry}\label{sec:loss_geometry}
We first examine how SL and RL shape the loss landscape from which adversarial gradients are derived. The fundamental difference traces back to the optimization process. In SL, the cross-entropy loss assigns each sample a deterministic gradient: given a fixed input-label pair $(X, y)$, the gradient $\nabla_x \mathcal{L}_{CE}$ is uniquely determined. During training, this deterministic signal creates a loss landscape with sharp, well-defined gradient structures, which is useful information for gradient-based attackers. In RL, the policy gradient objective introduces stochasticity through action sampling and $\varepsilon$-greedy exploration. This noisy optimization signal acts as an implicit regularizer, producing a flatter loss landscape with less pronounced gradient features, which is harder for gradient-based attackers to exploit. For further study, Figure~\ref{fig:SL_rl_decisionboundary_losslandscape} visualizes this difference on a representative CIFAR-10 image.
 
\textbf{Decision boundary:} Figure~\ref{fig:SL_rl_decisionboundary_losslandscape} (a) shows the decision boundary of both SL and RL, where SL's decision boundary should be steeper than RL's decision boundary, because of its deterministic gradients; however, a 2D decision boundary shows limited visual differentiation.

\textbf{Loss landscape:} Figure~\ref{fig:SL_rl_decisionboundary_losslandscape} (b) shows that SL has a larger gradient magnitude and a wider dynamic range than the RL on the decision boundary (\((max-min)_{boundary}\) value of SL \(>>\) 10.5, while \((max-min)_{boundary}\) value of RL \(<\) 10.5). This loss gradient difference directly impacts adversarial vulnerability: SL's large loss gradients enable efficient perturbation calculation via \(\nabla_{x} \mathcal{L}(x,y) \), whereas RL's small loss gradients inherently resist gradient-based attack optimization.

\begin{figure}[htbp]
    \centering
    \includegraphics[width=1.0\linewidth]{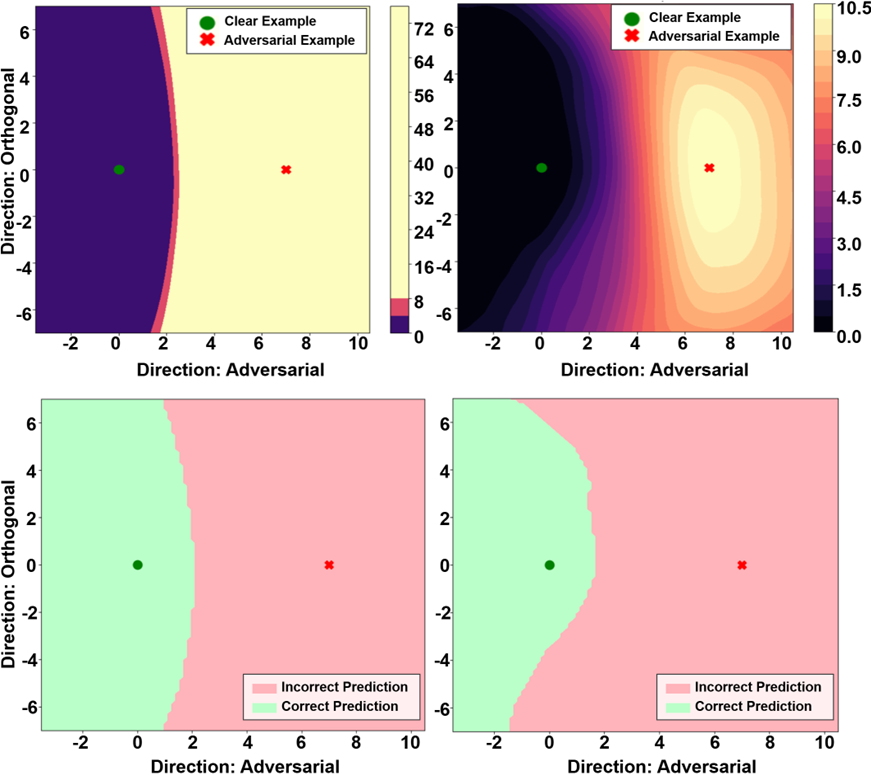}
    \caption{Comparative analysis of 6-layer CNN models trained with (left) SL versus (right) RL on one image from CIFAR-10: (upper) decision boundary and (bottom) loss landscape.}
    \label{fig:SL_rl_decisionboundary_losslandscape}
\end{figure}

\subsection{Gradient Instability Analysis}\label{sec:grad_instability}
We next quantify the gradient-level differences using both static and dynamic indicators, evaluated on the 6-layer CNN trained on CIFAR-10.
 
\textbf{Static indicators (the gradient field before attack):}
For the static indicators, the \textbf{average gradient norm (AGN)} is larger for SL (2.158) than for RL (1.9527) for the complete CIFAR-10 dataset. This indicates that small input perturbations cause larger loss change in SL, yielding larger increases in the attack loss at each step. The \textbf{input gradient variance (IGV)} further confirms that SL updates in input space are consistently larger than RL, shown in Figure~\ref{fig:cos_norm_eps3_digv_igv_cos_norm}~(a). In contrast, the \textbf{directional input gradient variance (dIGV)} is markedly higher for RL, shown in Figure~\ref{fig:cos_norm_eps3_digv_igv_cos_norm}~(a), reflecting greater instability of gradient directions under perturbations. It is concluded that these results imply that SL is more vulnerable because of its larger and more stable gradients (high AGN/IGV, low dIGV), whereas RL is more robust due to unstable gradient directions (high dIGV) and smaller increases in the attack loss at each step (low AGN/IGV).
 
\textbf{Dynamic indicators (gradient evolution during attack):}
For the dynamic indicators, Figure~\ref{fig:cos_norm_eps3_digv_igv_cos_norm}(b) tracks gradient behavior across PGD iterations. The \textbf{gradient stability under attack (GSUA)} remains high for SL ($\sim$0.8), indicating that the adversary maintains a coherent ascent direction throughout the attack. For RL, GSUA drops to low or even negative values ($\sim$$-0.2$), meaning that consecutive attack steps point in inconsistent or opposing directions, preventing the attack from making steady progress. The \textbf{$\ell_2$ gradient norm} trajectory tells a complementary story: SL maintains a large gradient norm ($\sim$0.6) across iterations, enabling substantial adversarial progress per step, while RL's gradient norm remains small ($\sim$0.1), further reducing the per-step loss gain. The combination of directional instability (low GSUA) and small gradient magnitude (low $\ell_2$ norm) creates a compounding effect: each PGD step on an RL model not only moves in an unreliable direction but also produces a smaller increase in the attack objective, tightening the upper bound on attack-loss growth in Proposition~\ref{prop:loss_growth} and making the attack harder to succeed within a fixed iteration budget.

\begin{figure}[htbp]
    \centering
    \includegraphics[width=1.0\linewidth]{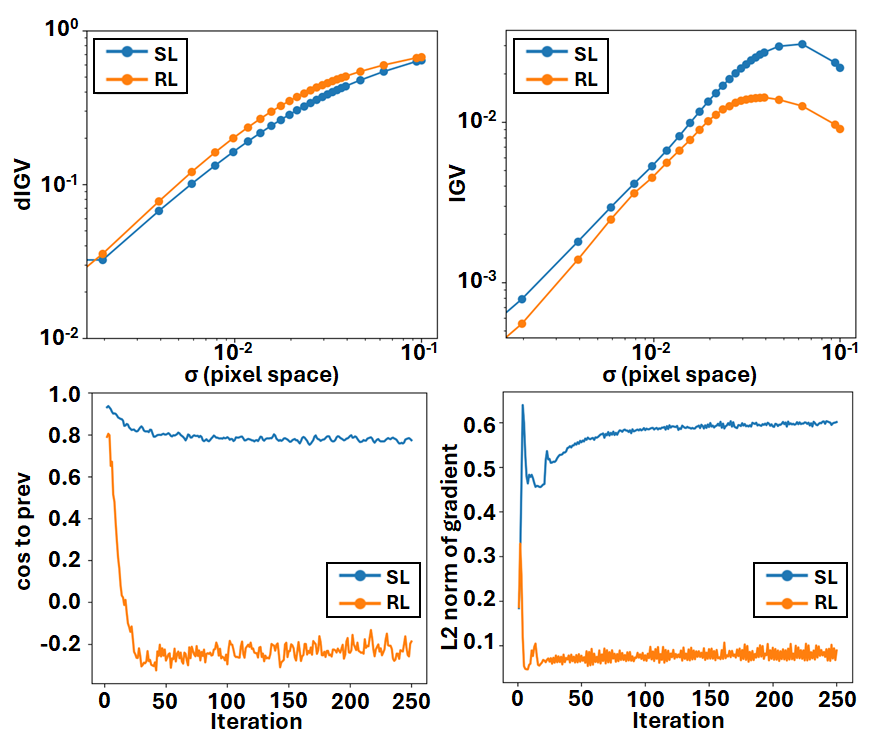}
    \caption{Comparison of 6-layer CNNs trained with SL (blue) and RL (orange) on CIFAR-10: (upper) Gradient variance indicators (dIGV, IGV) as a function of input noise scale $\sigma$ and (bottom) PGD attack dynamics showing cosine similarity to the previous step and gradient $L_2$ norm across iterations.}
    \label{fig:cos_norm_eps3_digv_igv_cos_norm}
\end{figure}

%\todo{TODO experiment: epsilon- non-epsilon during 1-250 iteration}

\subsection{Predictive Uncertainty Under Attack}\label{sec:analysis_entropy}
The gradient-level analysis predicts a specific consequence for model outputs: if RL's gradients resist adversarial optimization, then attacks should be less successful at driving the model toward confident misclassifications. We verify this prediction through predictive entropy analysis. Figure~\ref{fig:epsilon_entropy} shows the mean predictive entropy of SL (CNN-SL), adversarially-trained SL (CNN-SL-adv), and RL (CNN-RL) under varying perturbation magnitudes (\(\epsilon\)). From the results, the RL model consistently produces a higher entropy than both SL and SL-adv.

The explanation follows directly from the gradient analysis. SL is characterized by stable gradient directions (high cosine similarity, low dIGV) and relatively large gradient norms (high AGN, large \(L_2\) norm), allowing gradient-based attack (e.g., PGD) updates to efficiently align with adversarial directions and quickly drive an incorrect logit above the correct one. In contrast, RL has unstable gradient directions (low cosine similarity, high dIGV) and smaller gradient norms (low AGN, small \(L_2\) norm), meaning that attack steps tend to fluctuate in direction and produce smaller per-step gains in the attack objective. Even when attack steps move toward an adversarial direction, the increase of the incorrect logit relative to the correct one is much slower due to small gradient norms. As a result, SL tends to yield highly confident but incorrect predictions (e.g., [0.01, 0.01, 0.98], low entropy), as shown in Figure~\ref{fig:epsilon_entropy} (right), whereas RL outputs remain less sharply peaked (e.g., [0.3, 0.3, 0.4], higher entropy). These findings suggest that SL models are more prone to overconfident errors, while RL models maintain higher predictive uncertainty even when misclassifying.

\begin{figure}[htbp]
    \centering
    \includegraphics[width=0.9\linewidth]{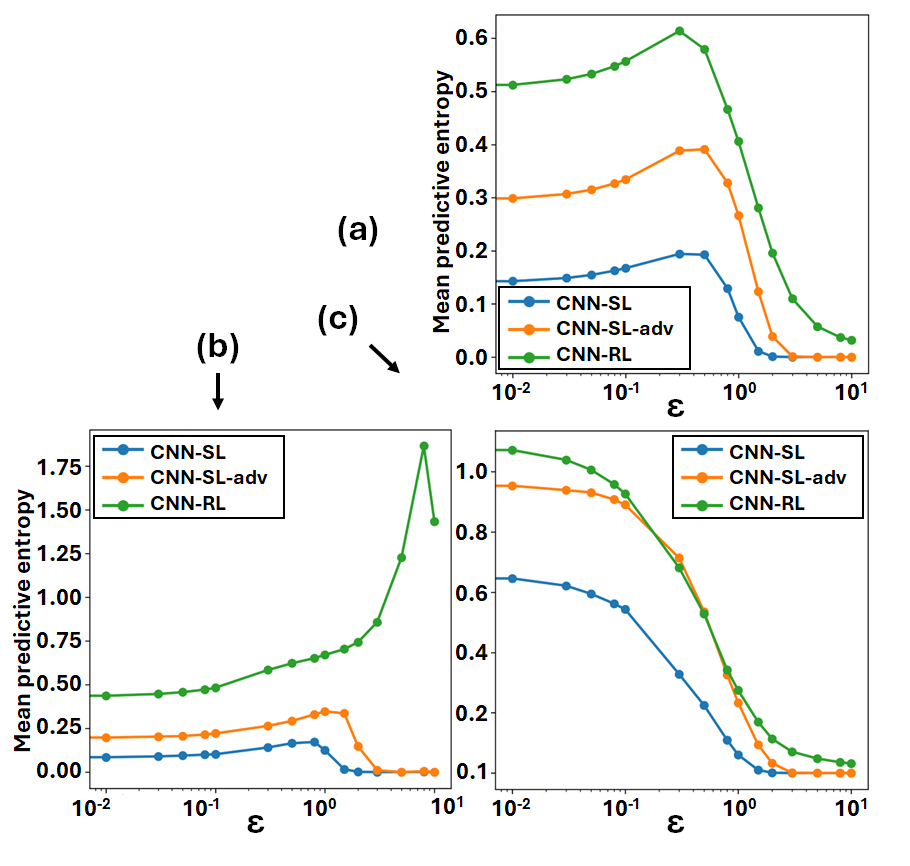}
    \caption{Predictive entropy vs. \(\epsilon\) (perturbation magnitude) for 6-layer CNN Architecture (CNN-SL (blue), CNN-SL-adv (orange), CNN-RL (green)) on CIFAR-10: total (a), correct prediction (b), wrong prediction (c).}
    \label{fig:epsilon_entropy}
\end{figure}

\section{Theoretical Grounding: From Empirical Indicators to Gradient-Based Attackability}\label{sec:formal_attackability}

The results in Section~\ref{sec:grad_instability} reveal two complementary gradient-level effects induced by RL training. First, the static indicators show that RL-trained models exhibit smaller input-gradient magnitudes than their SL counterparts. Second, the dynamic indicators show that RL-trained models exhibit substantially less stable attack-gradient directions along PGD trajectories. These two effects together suggest that RL reduces the effectiveness of gradient-based adversarial optimization. In this section, we will formalize this connection between gradient indicators and attack behavior. 

\subsection{Unified Setup}\label{sec:unified_setup}
Consider a clean example $(x_0,y)$ and a model $M$. Let $z_k^M(x)$ denote the logit for class $k$. Define the logit margin of the true class against the strongest wrong class by
\[
m_M(x,y):=z_y^M(x)-\max_{k\ne y}z_k^M(x).
\]
where $m_M(x,y)>0$ implies that $M$ predicts $y$ correctly, while $m_M(x,y)\le 0$ means that $x$ is on or beyond the decision boundary for an untargeted attack.

Let $J_M(x,y)$ be the attack objective maximized by the adversary, such as cross-entropy or a margin-based loss. We consider a $\ell_2$ gradient-based attack with step size $\alpha>0$ and $T$ steps:
\[
g_t:=\nabla_x J_M(x_t,y),
\qquad
 d_t:=
 \begin{cases}
 g_t/\norm{g_t}_2, & g_t\ne 0,\\
 0, & g_t=0,
 \end{cases}
\]
\[
x_{t+1}:=\Pi_C(x_t+\alpha d_t),
\qquad t=0,\ldots,T-1,
\]
where $g_t$ is the input gradient of the attack objective at the current PGD iterate, $d_t$ is the proposed normalized attack direction, $\alpha$ is the PGD step size, $x_t$ and $x_{t+1}$ are adversarial candidates at iteration $t$ and $t+1$, and $\Pi_C$ denotes Euclidean projection onto the feasible perturbation set
\[
C=B_2(x_0,\epsilon_{\mathrm{adv}})\cap[0,1]^d=\{x\in[0,1]^d:\|x-x_0\|_2\le \epsilon_{\mathrm{adv}}\}.
\]

Because projection may change both the step length and direction, define the actual attack step, its step length, and its unit direction:
\[
s_t:=x_{t+1}-x_t,
\qquad
\ell_t:=\norm{s_t}_2,
\qquad
u_t:=\frac{s_t}{\ell_t} \quad \text{whenever } \ell_t>0.
\]

For any intermediate PGD iteration $\tau\le T$, define the cumulative attack-gradient norm and the net PGD displacement up to iteration $\tau$, respectively, by
\[
\mathcal G_\tau(M;x_0,y):=\sum_{t=0}^{\tau-1}\norm{g_t}_2,
\qquad
R_\tau(M;x_0,y):=\norm{x_\tau-x_0}_2.
\]

\begin{assumption}[Local smoothness of the attack loss]
\label{assump:loss_smoothness}
The attack objective $J_M(\cdot,y)$ is $L_J$-smooth on perturbation set $C$, i.e.,
\[
\norm{\nabla_x J_M(a,y)-\nabla_x J_M(b,y)}_2
\le
L_J\norm{a-b}_2 
\quad \text{for all } a,b\in C.
\]
\end{assumption}

\begin{assumption}[Local Lipschitzness of the margin]
\label{assump:local_margin}
The margin $m_M(\cdot,y)$ is $L_m$-Lipschitz on perturbation set $C$, i.e.,
\[
|m_M(a,y)-m_M(b,y)|\le L_m\norm{a-b}_2
\quad \text{for all } a,b\in C.
\]
\end{assumption}
% \begin{remark}[On neural networks]
% For ReLU networks and max-margin functions, differentiability may fail at a measure-zero set of kink points. The statements can be interpreted under local smoothness almost everywhere, or with subgradients / generalized gradients. Since the purpose here is to formalize attack progress along realized PGD trajectories, these regularity assumptions should be understood as local assumptions on the relevant perturbation region.
% \end{remark}

All proofs of the propositions and theorem in this section are deferred to Appendix~\ref{app:theory_proofs}.

\subsection{Finite-budget first-order attackability}
\begin{proposition}[Pathwise attack-loss growth bound]
\label{prop:loss_growth}
Assume $J_M(\cdot,y)$ is $L_J$-smooth on $C$~\cite{sinha2017certifying}. Then for every intermediate PGD iteration $\tau\le T$ of the projected $\ell_2$-PGD trajectory,
\[
J_M(x_\tau,y)-J_M(x_0,y)
\le
\alpha\sum_{t=0}^{\tau-1}\norm{g_t}_2
+
\frac{L_J\alpha^2}{2}\tau.
\]
Equivalently,
\[
J_M(x_\tau,y)
\le
J_M(x_0,y)
+
\alpha\mathcal G_\tau(M;x_0,y)
+
\frac{L_J\alpha^2}{2}\tau.
\]
\end{proposition}

\begin{remark}[Interpretation]
This proposition claims that along a fixed finite-step PGD trajectory, a smaller cumulative attack-gradient norm gives a smaller upper bound on how much the attack objective can increase.
\end{remark}

\begin{proposition}[Attack-step interference identity]
\label{prop:step_interference}
For any projected PGD trajectory and any intermediate PGD iteration $\tau\le T$,
\[
x_\tau-x_0=\sum_{t=0}^{\tau-1}s_t.
\]
Moreover,
\[
R_\tau(M;x_0,y)^2
=
\norm{x_\tau-x_0}_2^2
=
\sum_{t=0}^{\tau-1}\ell_t^2
+
2\sum_{0\le i<j\le \tau-1}\ell_i\ell_j\inner{u_i}{u_j}.
\]
\end{proposition}

\begin{remark}[Interpretation]
Proposition~\ref{prop:step_interference} shows that PGD progress is not determined only by the length of individual steps, but also by how coherently those steps align over the attack trajectory. The cross terms $\inner{u_i}{u_j}$ measure pairwise directional coherence between realized PGD updates. When these terms are positive, attack steps accumulate coherently and the net displacement $R_\tau$ grows. When they are small or negative, successive updates interfere destructively, making the trajectory behave more like a random walk or self-canceling process. This formalizes the mechanism suggested by GSUA: SL maintains stable ascent directions, whereas RL produces unstable or opposing directions that reduce coherent finite-budget PGD progress.
\end{remark}

\begin{definition}[Weighted path coherence]
\label{def:weighted_path_coherence}
For any partial PGD trajectory ending at iteration $\tau$, define
\[
\mathcal C_\tau^{w}
=
\frac{
\sum_{0\le i<j\le \tau-1}
\ell_i\ell_j\inner{u_i}{u_j}
}{
\sum_{0\le i<j\le \tau-1}
\ell_i\ell_j
},
\]
whenever the denominator is nonzero.
\end{definition}

By Proposition~\ref{prop:step_interference},
\[
R_\tau(M;x_0,y)^2
=
\sum_{t=0}^{\tau-1}\ell_t^2
+
2\mathcal C_\tau^w
\sum_{0\le i<j\le \tau-1}\ell_i\ell_j.
\]
Thus, positive path coherence increases net PGD displacement, while low or negative path coherence reduces it. GSUA measures adjacent-step coherence, whereas $\mathcal C_\tau^w$ captures all-pair coherence along the full PGD trajectory; we therefore interpret GSUA as a dynamic proxy for the broader path-coherence effect.

\begin{theorem}[Finite-budget first-order attackability]
\label{thm:finite_budget_attackability}
Assume that $J_M(\cdot,y)$ is $L_J$-smooth on $C$~\cite{sinha2017certifying} and that 
$m_M(\cdot,y)$ is $L_m$-Lipschitz on $C$~\cite{tsuzuku2018lipschitz,madry2017towards}. Then for every intermediate 
PGD iteration $\tau\le T$, the following two bounds hold.

\noindent\textbf{Attack-loss growth bound.}
\[
J_M(x_\tau,y)
\le
J_M(x_0,y)
+
\alpha\mathcal G_\tau(M;x_0,y)
+
\frac{L_J\alpha^2}{2}\tau.
\]

\noindent\textbf{Margin-preservation bound.}
\[
m_M(x_\tau, y) \ge m_M(x_0, y)
- L_m \sqrt{
\sum_{t=0}^{\tau-1} \ell_t^2
+ 2 \sum_{0 \le i < j \le \tau-1}
\ell_i \ell_j \langle u_i, u_j \rangle
}.
\]
Equivalently,
\[
m_M(x_\tau,y)
\ge
m_M(x_0,y)
-
L_mR_\tau(M;x_0,y).
\]
\end{theorem}

\begin{remark}[How to compare RL and SL]
The theorem gives measurable quantities for comparing finite-budget attackability. Under comparable local smoothness and initial margins, a model with smaller cumulative attack-gradient norm $\mathcal G_\tau$ has a smaller attack-loss growth bound. A model with smaller net displacement $R_\tau$, which can arise when realized attack-step directions have low or negative path coherence, has a smaller upper bound on margin reduction. Thus, if RL exhibits smaller $\mathcal G_\tau$ and smaller realized displacement $R_\tau$ (consistent with lower path coherence) than SL under the same attack configuration, the theorem predicts tighter finite-budget first-order attackability bounds for RL.
\end{remark}
\section{Discussion}\label{sec:discussion}
The mechanism analysis in Section~\ref{sec:analysis} establishes that RL training produces fundamentally different gradient properties than SL. Specifically, RL has smaller norms, unstable directions, and flatter loss landscapes, coming from the stochastic nature of policy gradient optimization with $\varepsilon$-greedy exploration. In this section, we synthesize these findings into three practical insights about when and why RL-based training helps, where it falls short, and how it can be combined with adversarial training for stronger defense. We also discuss the computational trade-offs involved.

\subsection{Why RL disrupts gradient-based attacks.}
SL's cross-entropy loss provides each training sample with a deterministic, low-variance gradient signal. Over training, this produces models with sharp loss landscapes, large input-gradient norms (high AGN/IGV), and stable gradient directions (low dIGV, high GSUA). These properties are precisely what gradient-based attackers need: a reliable, informative gradient field that supports efficient adversarial optimization. However, RL's policy gradient objective, combined with $\varepsilon$-greedy exploration, introduces stochasticity into the training signal. This acts as an implicit form of regularization, producing flatter loss landscapes, smaller gradient norms, and highly unstable gradient directions. When an attacker attempts to optimize adversarial perturbations on such a model, each PGD step faces both an unreliable direction and a small per-step loss gain, causing the attack to fail within practical iteration budgets. From a security perspective, this means that deploying an RL-trained model raises the bar for gradient-based attackers, requiring them to invest in more sophisticated attack strategies (adaptive restarts, alternative loss objectives, or surrogate-based approaches) rather than relying on standard PGD. This is a meaningful practical benefit, as PGD remains the default attack in many real-world adversarial evaluation pipelines.
 
\subsection{Why RL alone is insufficient.}
Despite disrupting gradient-based optimization, RL does not eliminate adversarial regions in the input space, and it only makes them harder to find via gradient methods. This is evidenced by two findings: (1)~AutoAttack, which uses adaptive strategies beyond standard gradients, reduces the RL advantage to near zero compared to SL; and (2)~transfer-based attack that adversarial examples generated on SL surrogates transfer effectively to RL models, confirming that exploitable input regions persist. These results indicate that RL training disrupts the optimization pathway to adversarial examples but does not eliminate the adversarial regions in input space. The adversarial vulnerability of neural networks is not solely a property of gradient quality; it also stems from the geometry of decision boundaries relative to natural data points~\cite{ilyasAdversarialExamplesAre2019}. RL's exploration-induced gradient flattening addresses the former but not the latter.
 
\begin{comment}
    This distinction has important implications for how RL-based robustness should be characterized. Following the taxonomy of Athalye et al.~\cite{athalyeObfuscatedGradientsGive2018}, RL's gradient disruption shares surface-level similarity with gradient masking. However, unlike typical gradient masking defenses, RL-adv maintains its robustness advantage even under AutoAttack (36.27\% vs.\ 24.87\% for SL-adv), suggesting that the gradient disruption is not merely obscuring a brittle defense but genuinely complements adversarial training.
\end{comment}
 
\subsection{Why RL-adv provides the strongest defense.}
RL-adv combines two complementary robustness mechanisms operating at different levels:
\begin{itemize}
    \item \textbf{Gradient-level defense} (from RL training): degrades the quality of gradient information available to attackers, increasing the difficulty of adversarial optimization. 
    \item \textbf{Decision-boundary-level defense} (from adversarial training): pushes decision boundaries away from data points, reducing the existence and proximity of adversarial regions in input space.
\end{itemize}
 
Specifically, neither mechanism alone is sufficient. RL without adversarial training is vulnerable to transfer attacks (8.29\% accuracy against SL-sourced adversarial examples). SL-adv without RL's gradient disruption is more easily attacked by standard gradient methods (4.96\% under PGD vs.\ 48.63\% for RL-adv on CIFAR-10). Their combination in RL-adv yields the strongest robustness across all attack types: 36.27\% under APGD-CE, 30.88\% against SL-sourced transfer attacks, and 32.41\% under Square Attack. This dual-layer defense framework suggests a practical design principle: rather than relying solely on adversarial training to harden decision boundaries, we can additionally leverage the training paradigm itself to degrade the gradient information available to attackers. We also note that the RL-adv margin over SL-adv narrows on more complex datasets under certain AutoAttack components (Appendix~\ref{app:aa_results}), suggesting that fully realizing the dual-layer benefit may require task-specific tuning of the adversarial training configuration.

\subsection{Computational Cost}
A practical limitation of RL-based training is its computational overhead. RL requires approximately 20$\times$ more training epochs than SL (2000 vs.\ 100) to achieve convergence, due to the higher variance of the policy gradient estimator and the exploration overhead introduced by $\varepsilon$-greedy action selection. On a single NVIDIA A100 GPU, training a 6-layer CNN on CIFAR-10 takes approximately 20 minutes for SL and approximately 300 minutes for RL. This overhead becomes increasingly prohibitive for larger models and datasets. For example, RL training on Places-365 under our settings is estimated to require multiple weeks on a single A100, which prevented us from completing those experiments.
 
Several directions could mitigate this limitation. (1)~First, variance reduction techniques for policy gradient estimators (e.g., learned baselines, control variates~\cite{schulman2015high, gu2016q}) could accelerate RL convergence. (2)~Second, hybrid training schedules that initialize with SL and fine-tune with RL could reduce the total RL training budget while preserving gradient-disruption benefits~\cite{ranzato2015sequence, ramrakhya2023pirlnav}. (3)~Third, more sample-efficient RL algorithms (e.g., PPO with adaptive clipping~\cite{schulman2017proximal}) could be explored as alternatives to REINFORCE. We leave the systematic investigation of these efficiency improvements to future work, noting that the computational cost does not diminish the conceptual contribution of identifying exploration-induced gradient disruption as a robustness mechanism.

\section{Conclusion}\label{sec:conclusion}
We investigated how the choice of training paradigm (supervised learning versus reinforcement learning) affects the adversarial robustness of image classifiers. Through extensive experiments across three datasets (CIFAR-10, CIFAR-100, ImageNet-100) and multiple architectures (4/6-layer-CNN, ResNet-18), we identified a mechanism by which RL training disrupts gradient-based adversarial attacks: the stochastic optimization signal introduced by policy gradients and $\varepsilon$-greedy exploration produces models with smaller gradient norms, highly unstable gradient directions, and flatter loss landscapes, making it substantially harder for attackers to optimize adversarial perturbations via standard gradient methods.
 
However, our analysis also reveals a critical limitation: RL training disrupts the optimization process used to discover adversarial examples, but does not eliminate the adversarial regions themselves. As a result, RL-trained models remain vulnerable to transfer-based attacks from SL surrogates and adaptive attacks that bypass gradient information. To overcome this limitation, we proposed RL-adv, which combines RL's gradient-level defense with the decision-boundary-level defense provided by adversarial training. RL-adv achieves the strongest overall robustness across attack types evaluated, demonstrating that these two mechanisms are both important for stronger robustness. This dual-layer defense framework offers a practical design principle: robustness can be improved not only by hardening decision boundaries through adversarial training, but also by leveraging the training paradigm itself to degrade the gradient information available to attackers. We believe these findings provide a stepping stone toward understanding how training paradigms shape adversarial vulnerability, with potential implications beyond image classification.
\section{Limitation and Future Work}\label{sec:limitation_future_work}
Although our study provides a systematic understanding of why reinforcement learning improves adversarial robustness against gradient-based attacks, several limitations remain, each suggesting concrete directions for future research.
 
\textbf{Transfer vulnerability.}
While RL-trained models resist direct gradient-based attacks, they remain vulnerable to adversarial examples transferred from SL models (Section~\ref{sec:transfer_based_attack}). This indicates that adversarial regions still exist in the input space; RL training primarily disrupts the optimization process used to discover them rather than eliminating these regions entirely. Future work should investigate whether stronger forms of RL regularization (e.g., entropy bonuses, policy smoothing) can reduce the density of adversarial regions, not just the accessibility of gradients.
 
\textbf{Computational overhead.}
RL training requires approximately 20$\times$ more epochs than SL (2000 vs.\ 100) to achieve convergence, making it computationally expensive for large-scale applications. Several promising directions could mitigate this cost: (1)~variance reduction techniques for policy gradient estimators~\cite{schulman2015high, gu2016q} could accelerate RL convergence; (2)~hybrid training schedules that initialize with SL and fine-tune with RL~\cite{ranzato2015sequence, ramrakhya2023pirlnav} could reduce the total RL training budget while preserving gradient-disruption benefits; and (3)~more sample-efficient RL algorithms such as PPO~\cite{schulman2017proximal} could replace REINFORCE.
 
\textbf{Architecture scope.}
Our robustness analysis focuses primarily on convolutional architectures (CNNs and ResNet-18). Extending RL-based robustness methods to Transformer or hybrid architectures requires a separate experimental setup. Whether the gradient-disruption mechanism we identified transfers to attention-based models, where the gradient structure differs fundamentally from CNNs, remains an open question.
 
\textbf{Attack coverage.}
We primarily evaluate robustness against standard attacks (PGD, AutoAttack, Square Attack). Stronger adaptive attacks specifically designed to exploit RL-trained models (e.g., attacks that account for the stochastic training process or target the flatter loss landscape directly) may reveal additional vulnerabilities. Future work should also investigate whether certified robustness guarantees can be established for RL-trained models, providing formal bounds rather than empirical evaluations alone.

\bibliographystyle{ACM-Reference-Format}
\bibliography{references}

\appendix

\section{Disambiguation of Epsilons}\label{app:eps_disambiguation}
To avoid ambiguity, we distinguish three different epsilon symbols used throughout the paper: the exploration rate in RL, the adversarial perturbation budget, and a tiny numerical constant for stability.
\begin{itemize}
\item \(\varepsilon_{\text{greedy}}\): \emph{Exploration rate} in \(\varepsilon\)-greedy behavior policy (probability of sampling a random action).
\item \(\varepsilon_{\text{adv}}\): \emph{Adversarial perturbation budget} (radius of the \(\ell_p\)-ball used by the attacker).
\item \(\epsilon_{\text{num}}\): Tiny numeric constant inside \(\log\) for stability (e.g., \(10^{-8}\)); never to be confused with the above.
\end{itemize}

\section{Robustness Indicators}\label{app:robustness_indicators}
This subsection summarizes the diagnostic indicators used to characterize robustness from the perspectives of decision geometry, loss landscape structure, gradient magnitude, gradient directional stability, and predictive uncertainty.

\textbf{Decision boundary diagrams} illustrate the classification regions under adversarial attack, and 

\textbf{loss landscape diagrams} visualize the loss gradient information. Both diagrams are drawn under two gradient directions: standard adversarial attack gradient and orthogonal-direction attacks gradient to draw a 2D diagram, where perturbations are constrained to be perpendicular to the gradient ascent direction (\(\nabla_{x}\mathcal{L} \perp \delta\)). 

\textbf{IGV (Input Gradient Variance)} measures the variance of the input gradient under small Gaussian perturbations:

\begin{equation}
\label{eq:igv}
IGV = \mathbb{E}_{x \sim \mathcal{D}} \left \{ \mathbb{E}_{\epsilon \sim \mathcal{N}(0, \sigma^2)} \left [ Var(\nabla_{x} \mathcal{L}(x + \epsilon, \hat{y})) \right ] \right \} ,
\end{equation}
where \(\epsilon\) is the Gaussian noise added to the input sample, \(x\) is the input sample, \(\hat{y}\) is the predicted sample, \(Var(\cdot)\) is the gradient variance, \(\nabla_{x} \mathcal{L}(\cdot)\) is the gradient.

\textbf{dIGV (direction Input Gradient Variance)} calculates the direction of gradient according to the input variance:

\begin{equation}
\label{eq:digv}
\begin{aligned}
\mathrm{dIGV} &= \mathbb{E}_{x \sim \mathcal{D}} \bigg\{ 
    \mathbb{E}_{\epsilon \sim \mathcal{N}(0, \sigma^2)} \left[ 
    1 - \left\langle 
        \frac{g^{(\epsilon)}}{\|g^{(\epsilon)}\|_2}, 
        \bar{u}
    \right\rangle 
    \right] \bigg\}, \\
\bar{u} &= \frac{\mathbb{E}_{\epsilon}\left( \frac{g^{(\epsilon)}}{\|g^{(\epsilon)}\|_2} \right)}
    {\left\| \mathbb{E}_{\epsilon}\left( \frac{g^{(\epsilon)}}{\|g^{(\epsilon)}\|_2} \right) \right\|_2}, \\
g^{(\epsilon)} &= \nabla_x \mathcal{L}(x+\epsilon, \hat{y}) ,
\end{aligned}
\end{equation}
where \(g\) is the attack gradient. 

\textbf{AGN (Average Gradient Norm)} calculates the sensitivity of gradient according to the input variance:

\begin{equation}
\label{eq:agn}
AGN = \mathbb{E}_{x \sim \mathcal{D}}[||(\nabla_{x} \mathcal{L}(x, y)||_{2}] ,
\end{equation}
where \(y\) is the true label, \(\mathcal{L}(\cdot)\) is the loss function.

\textbf{Gradient stability under attack (GSUA) diagram} is drawn to visualize the gradient stability under attack. It is calculated by calculating the cosine (similarity) between two steps of attack gradient:

\begin{equation}
\label{eq:gsua}
    \text{GSUA}^{(t)} = \cos\theta^{(t)} = \frac{g^{(t)} \cdot g^{(t-1)}}{\|g^{(t)}\| \|g^{(t-1)}\|}, g^{(t)} = \nabla_x \mathcal{L}(x^{(t)}, y) ,
\end{equation}
where \(t\) is the step. 

%\subsubsection{Entropy (Overconfidence) Measurement}
\textbf{Mean predictive entropy \(H\)} is used to represent the dispersibility of predicted output:

\begin{equation}
\label{eq:mean_entropy}
H = \frac{1}{N}\sum_{i=1}^N \left[ -\sum_{k=1}^C p_k(\mathbf{x}_i) \cdot \log p_k(\mathbf{x}_i) \right] ,
\end{equation}
where \(N\) is the number of test samples, \(C\) is the number of classes, \(p_k(\mathbf{x}_i)\) denotes the predicted probability of class \(k\) on input \(x_{i}\) (either clean or adversarial).

\section{Model Architectures}\label{app:models}
For completeness, we provide the backbone architectures used in the main experiments. These diagrams clarify the structural differences among the 4-layer CNN, 6-layer CNN, and ResNet-18 models considered throughout the paper, as shown in Figure~\ref{fig:model_architectures_appendix}.

\begin{figure}[htbp]
\centering
\includegraphics[width=1.0\linewidth]{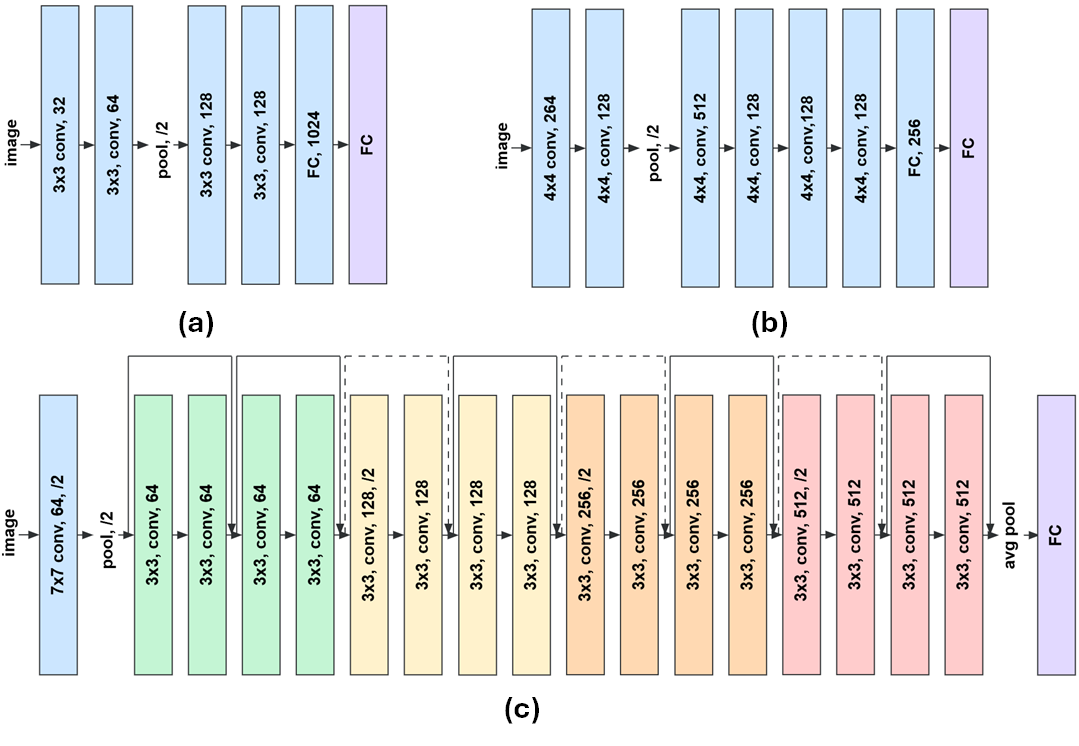}
\caption{Model architectures: (a) 4-layer CNN, (b) 6-layer CNN, and (c) ResNet-18 with residual connections.}
\label{fig:model_architectures_appendix}
\end{figure}

\section{Perturbation Budget}\label{app:perturbation_budget}
To verify that our findings are not specific to a single perturbation budget, we evaluate the 6-layer CNN on CIFAR-10 under $\ell_2$ PGD across seven budget levels. Table~\ref{tab:eps_sweep} reports adversarial accuracy for SL and RL models.

RL consistently outperforms SL across all budgets. The gap is negligible at $\epsilon = 0.25$ but widens rapidly, exceeding 50 percentage points for $\epsilon \geq 2.0$. RL accuracy degrades gracefully (from 74.98\% to 55.77\%), whereas SL collapses (from 74.88\% to 5.00\%). This pattern is consistent with our mechanism analysis: as the perturbation budget grows, the attacker has greater capacity to exploit informative gradient signals. SL's sharp, stable gradients become increasingly exploitable, while RL's unstable, low-magnitude gradients continue to resist optimization across the enlarged search space. We select $\epsilon = 7.0$ as the primary evaluation budget throughout the paper to operate in this strong-attack regime where gradient-structural differences are most clearly exposed.

\begin{table}[htbp]
\centering
\caption{Adversarial accuracy (\%) under $\ell_2$ PGD across perturbation budgets (6-layer CNN, CIFAR-10, 250 iterations). Images are normalized to $[0,1]$.}
\label{tab:eps_sweep}
\resizebox{\linewidth}{!}{
\begin{tabular}{llllllll}
\hline
$\epsilon$   & 0.25  & 0.5   & 1.0   & 2.0   & 3.5   & 5.0   & 7.0   \\ \hline
SL  (\%) & 74.88 & 53.86 & 25.37 & 17.59 & 14.26 & 11.21 & 5.00  \\
RL  (\%) & 74.98 & 71.17 & 70.11 & 69.05 & 65.52 & 60.75 & 55.77 \\
GAP (\%) & 0.10  & 17.31 & 44.74 & 51.46 & 51.26 & 49.54 & 50.77 \\ \hline
\end{tabular}
}
\end{table}

\section{Norm Sensitivity Analysis}\label{app:norm_sensitivity_analysis}
This subsection tests whether the robustness advantage of RL depends on the attack norm, and in particular whether the effect weakens when the attacker uses only gradient signs rather than full gradient vectors. A key prediction of our mechanism analysis is that RL's gradient disruption should be most effective against attacks that rely on precise gradient information. To test this, we compare robustness under $\ell_\infty$ PGD, which applies a sign operation that discards gradient magnitude and reduces sensitivity to directional instability. Table~\ref{tab:norm_sensitivity} reports results for all four model variants on CIFAR-10.

Under $\ell_\infty$ PGD, the SL--RL gap without adversarial training is dramatically smaller than under $\ell_2$ (Table~\ref{tab:eps_sweep}), never exceeding 4 percentage points. At $\epsilon = 8/255$, SL achieves 13.35\% and RL achieves 15.76\%, compared to 5.00\% vs.\ 55.77\% under $\ell_2$ at $\epsilon = 7.0$. This confirms that RL's defense operates at the gradient-information level: when the sign operation strips away the gradient properties that RL disrupts, the robustness advantage largely disappears.

However, RL-adv remains the strongest model under $\ell_\infty$ as well (38.49\% vs.\ 25.08\% for SL-adv at $\epsilon = 8/255$). This demonstrates that while RL's gradient-level defense alone is norm-dependent, the dual-layer defense provided by RL-adv generalizes across norm types: adversarial training provides boundary-level protection that remains effective regardless of whether the attacker uses full gradient vectors ($\ell_2$) or only gradient signs ($\ell_\infty$).

\begin{table}[htbp]
\centering
\caption{Adversarial accuracy (\%) under $\ell_\infty$ PGD across perturbation budgets (6-layer CNN, CIFAR-10, 100 iterations). Images are normalized to $[0,1]$.}
\label{tab:norm_sensitivity}
\begin{tabular}{lllll}
\hline
$\epsilon$  & 2/255 & 4/255 & 8/255 & 16/255 \\ \hline
SL     (\%) & 72.98 & 47.93 & 13.35 & 0.29   \\
SL-adv (\%) & 77.59 & 58.59 & 25.08 & 1.44   \\
RL     (\%) & 71.88 & 49.28 & 15.76 & 4.18   \\
RL-adv (\%) & 77.81 & 64.83 & 38.49 & 8.49   \\ \hline
\end{tabular}
\end{table}

\section{Complete Model Performance}\label{app:model_performance}
This subsection collects the full clean and adversarial performance results across datasets and architectures, together with transfer-based evaluations that complement the main-paper analysis.

We evaluate the robustness of 4-layer CNN, 6-layer CNN, ResNet-18 on CIFAR-10, CIFAR-100, and ImageNet-100 using non-targeted $\ell_2$ PGD (\(K=250\), step size \(\alpha=0.3\), \(\epsilon=7.0\)). Unless otherwise noted, all results are performed on the test set. Perturbations are applied in the input space before normalization, where we attack the pre-normalized images and then apply dataset normalization. Adversarial examples are clipped to the valid image range.

\subsection{CIFAR-10}
The complete performance on CIFAR-10 across 4-layer CNN, 6-layer CNN, and ResNet-18 is shown in Table~\ref{tab:cifar10_model_performance_across_models}. Transfer analyses for these architectures are summarized in Table~\ref{tab:cifar10_transfer_analysis_4layerCNN}, Table~\ref{tab:cifar10_transfer_analysis_6layerCNN}, Table~\ref{tab:cifar10_transfer_resnet18_nopt} and Table~\ref{tab:cifar10_transfer_resnet18_pt}, respectively.

We note that the SL-trained ResNet-18 already achieves substantial adversarial accuracy (46.08\% without pretraining, 67.61\% with pretraining), far exceeding the 4/6-layer CNNs ($\sim$5--6\%). This is consistent with prior observations that deeper architectures with residual connections tend to learn smoother representations even under standard training~\cite{heDeepResidualLearning2016, madryDeepLearningModels2018}. Importantly, this \emph{architectural} implicit robustness is orthogonal to the \emph{training-paradigm} effect analyzed in this paper. The former arises from model capacity and structural properties (e.g., skip connections), while the latter arises from the stochastic optimization signal of RL. Evidence for their independence is visible in the transfer results (Tables~\ref{tab:cifar10_transfer_resnet18_nopt}--\ref{tab:cifar10_transfer_resnet18_pt}): SL-sourced adversarial examples still transfer effectively to SL-trained ResNet-18 (accuracy drops to 46.08\% under self-attack), confirming that its implicit robustness does not eliminate adversarial regions. In contrast, RL-trained ResNet-18 maintains 75.44\% adversarial accuracy under direct PGD, reflecting the additional gradient-disruption mechanism. The main text focuses on the 6-layer CNN precisely to isolate the training-paradigm effect in a setting where architectural implicit robustness is minimal.

\begin{table}[htbp]
\caption{Model robustness evaluation (train, test, and AEs) in CIFAR-10 datasets across models, evaluation under \attacksetting.}
\label{tab:cifar10_model_performance_across_models}
\begin{center}
\begin{tabular}{lccc}
\toprule
Model & Clean train (\%) & Clean test (\%) & AE (\%) \\
\hline
(SL) 4-layer-CNN        & 86.53 & 82.03 & 5.90  \\
(SL) 4-layer-CNN-adv    & 84.98 & 81.25 & 7.01  \\
(RL) 4-layer-CNN        & 88.85 & 83.07 & 36.66 \\
(RL) 4-layer-CNN-adv    & 88.57 & 82.94 & 35.67 \\
\hline
(SL) 6-layer-CNN        & 98.30 & 90.74 & 5.00  \\
(SL) 6-layer-CNN-adv    & 96.81 & 90.11 & 4.96  \\
(RL) 6-layer-CNN        & 95.93 & 88.50 & 55.77 \\
(RL) 6-layer-CNN-adv    & 94.12 & 87.63 & 48.63 \\
\hline
(SL) Resnet18           & 99.91 & 91.56 & 46.08 \\
(SL) Resnet18-pt        & 99.99 & 95.94 & 67.61 \\
(SL) Resnet18-adv       & 99.35 & 90.84 & 28.47 \\
(SL) Resnet18-pt-adv    & 99.97 & 95.56 & 56.96 \\
(RL) Resnet18           & 98.03 & 91.40 & 75.44 \\
(RL) Resnet18-pt        & 98.21 & 94.22 & 65.07 \\
(RL) Resnet18-adv       & 97.15 & 90.40 & 70.09 \\
(RL) Resnet18-pt-adv    & 97.38 & 93.73 & 68.40 \\
\hline
\end{tabular}
\end{center}
\end{table}

\begin{table}[htbp]
\caption{Transfer analysis in CIFAR-10 datasets on 4-layer-CNN, evaluation under \attacksetting.}
\label{tab:cifar10_transfer_analysis_4layerCNN}
\begin{center}
\resizebox{\linewidth}{!}{%
\begin{tabular}{lcccc}
\toprule
Model & (SL) 4-layer-CNN & (SL) 4-layer-CNN-adv & (RL) 4-layer-CNN & (RL) 4-layer-CNN-adv \\
\hline
(SL) 4-layer-CNN        & 5.89  & 8.20  & 34.79 & 35.20  \\
(SL) 4-layer-CNN-adv    & 6.82  & 6.95  & 35.33 & 35.66  \\
(RL) 4-layer-CNN        & 7.25  & 8.81  & 35.21 & 34.48  \\
(RL) 4-layer-CNN-adv    & 13.09 & 14.40 & 34.99 & 33.74  \\
\hline
(SL) 6-layer-CNN        & 14.09 & 14.27 & 35.20 & 35.30  \\
(SL) 6-layer-CNN-adv    & 15.16 & 15.87 & 36.48 & 37.07  \\
(RL) 6-layer-CNN        & 13.44 & 13.58 & 35.16 & 35.57  \\
(RL) 6-layer-CNN-adv    & 32.24 & 33.48 & 47.53 & 45.59  \\
\hline
\end{tabular}
}
\end{center}
\end{table}

\begin{table}[htbp]
\caption{Transfer analysis in CIFAR-10 datasets on 6-layer-CNN, evaluation under \attacksetting.}
\label{tab:cifar10_transfer_analysis_6layerCNN}
\centering
\resizebox{\linewidth}{!}{
\begin{tabular}{lcccc}
\toprule
Model & (SL) 6-layer-CNN & (SL) 6-layer-CNN-adv & (RL) 6-layer-CNN & (RL) 6-layer-CNN-adv \\
\hline
(SL) 6-layer-CNN        & 5.81 & 7.18 & 48.45 & 43.15 \\
(SL) 6-layer-CNN-adv    & 9.53 & 5.74 & 50.14 & 44.86 \\
(RL) 6-layer-CNN        & 8.29 & 7.92 & 48.84 & 42.69 \\
(RL) 6-layer-CNN-adv    & 30.88 & 22.56 & 54.31 & 46.91 \\
\hline
(SL) 4-layer-CNN        & 37.40 & 30.03 & 56.19 & 50.09 \\
(SL) 4-layer-CNN-adv    & 36.32 & 29.19 & 56.00 & 50.51 \\
(RL) 4-layer-CNN        & 32.23 & 25.94 & 56.65 & 49.92 \\
(RL) 4-layer-CNN-adv    & 38.86 & 31.22 & 56.54 & 49.30 \\
\hline
\end{tabular}%
}
\end{table}

\begin{table}[htbp]
\caption{Transfer analysis in CIFAR-10 on ResNet-18 (non-pretrained), evaluation under \attacksetting.}
\label{tab:cifar10_transfer_resnet18_nopt}
\centering
\resizebox{\linewidth}{!}{%
\begin{tabular}{lcccc}
\toprule
Model & (SL) Resnet18 & (SL) Resnet18-adv & (RL) Resnet18 & (RL) Resnet18-adv \\
\hline
(SL) Resnet18        & 46.08 & 65.09 & 81.07 & 89.54 \\
(SL) Resnet18-adv    & 80.05 & 28.47 & 85.14 & 89.01 \\
(RL) Resnet18        & 66.93 & 63.77 & 75.44 & 86.82 \\
(RL) Resnet18-adv    & 81.89 & 75.34 & 80.04 & 70.09 \\
\bottomrule
\end{tabular}
}
\end{table}

\begin{table}[htbp]
\caption{Transfer analysis in CIFAR-10 on ResNet-18 (pretrained), evaluation under \attacksetting.}
\label{tab:cifar10_transfer_resnet18_pt}
\centering
\resizebox{\linewidth}{!}{%
\begin{tabular}{lcccc}
\toprule
Model & (SL) Resnet18-pt & (SL) Resnet18-pt-adv & (RL) Resnet18-pt & (RL) Resnet18-pt-adv \\
\hline
(SL) Resnet18-pt     & 67.61 & 58.02 & 65.53 & 84.48 \\
(SL) Resnet18-pt-adv & 73.84 & 56.96 & 77.56 & 88.86 \\
(RL) Resnet18-pt     & 67.68 & 59.25 & 65.07 & 79.69 \\
(RL) Resnet18-pt-adv & 75.32 & 65.08 & 73.26 & 68.40 \\
\bottomrule
\end{tabular}
}
\end{table}

\subsection{CIFAR-100}\label{app:cifar_100}
The complete performance on CIFAR-100 across 4-layer CNN, 6-layer CNN, and ResNet-18 is shown in Table~\ref{tab:cifar100_model_performance_across_models}. Transfer analyses for these architectures are summarized in Table~\ref{tab:cifar100_transfer_analysis_4layerCNN}, Table~\ref{tab:cifar100_transfer_analysis_6layerCNN}, Table~\ref{tab:cifar100_transfer_resnet18_nopt} and Table~\ref{tab:cifar100_transfer_resnet18_pt}, respectively. We note that RL achieves substantially lower clean accuracy on the 4-layer CNN for CIFAR-100 (28.38\% vs.\ 52.40\% for SL). This is attributable to the limited model capacity of the 4-layer architecture for a 100-class task: the high-variance policy gradient signal requires sufficient model expressiveness to converge effectively. On the 6-layer CNN and ResNet-18, where model capacity is adequate, the clean accuracy gap remains within 5--8 percentage points.

\begin{table}[htbp]
\caption{Model robustness evaluation (train, test, and AEs) in CIFAR-100 datasets across models, evaluation under \attacksetting.}
\label{tab:cifar100_model_performance_across_models}
\begin{center}
\begin{tabular}{lcc}
\toprule
Model & Clean test (\%) & AE (\%) \\
\hline
(SL) 4-layer-CNN & 52.40 & 3.80 \\
(SL) 4-layer-CNN-adv & 50.77 & 4.06 \\
(RL) 4-layer-CNN & 28.38 & 15.06 \\
(RL) 4-layer-CNN-adv & 29.89 & 17.13 \\
\hline
(SL) 6-layer-CNN & 64.75 & 2.53 \\
(SL) 6-layer-CNN-adv & 63.61 & 2.83 \\
(RL) 6-layer-CNN & 59.80 & 13.06 \\
(RL) 6-layer-CNN-adv & 56.54 & 25.51 \\
\hline
(SL) Resnet18 & 69.36 & 14.83 \\
(SL) Resnet18-pt & 80.76 & 30.65 \\
(SL) Resnet18-adv & 68.76 & 12.61 \\
(SL) Resnet18-pt-adv & 79.70 & 26.55 \\
(RL) Resnet18 & 67.08 & 32.15 \\
(RL) Resnet18-pt & 77.22 & 45.91 \\
(RL) Resnet18-adv & 65.07 & 29.51 \\
(RL) Resnet18-pt-adv & 74.68 & 41.93 \\
\hline
\end{tabular}
\end{center}
\end{table}

\begin{table}[htbp]
\caption{Transfer analysis in CIFAR-100 datasets on 4-layer-CNN, evaluation under \attacksetting.}
\label{tab:cifar100_transfer_analysis_4layerCNN}
\begin{center}
\resizebox{\linewidth}{!}{
\begin{tabular}{lcccc}
\toprule
Model & (SL) 4-layer-CNN & (SL) 4-layer-CNN-adv & (RL) 4-layer-CNN & (RL) 4-layer-CNN-adv \\
\hline
(SL) 4-layer-CNN & 0.02 & 1.91 & 21.65 & 23.74 \\
(SL) 4-layer-CNN-adv & 3.16 & 0.04 & 23.54 & 25.18 \\
(RL) 4-layer-CNN & 5.37 & 6.17 & 10.48 & 11.89 \\
(RL) 4-layer-CNN-adv & 5.68 & 6.26 & 10.81 & 12.40 \\
\hline
(SL) 6-layer-CNN & 7.47 & 9.47 & 27.57 & 26.62 \\
(SL) 6-layer-CNN-adv & 7.85 & 9.93 & 27.82 & 27.03 \\
(RL) 6-layer-CNN & 5.69 & 8.30 & 24.16 & 23.64 \\
(RL) 6-layer-CNN-adv & 15.60 & 16.73 & 30.88 & 30.10 \\

\hline
\end{tabular}
}
\end{center}
\end{table}

\begin{table}[htbp]
\caption{Transfer analysis in CIFAR-100 datasets on 6-layer-CNN, evaluation under \attacksetting.}
\label{tab:cifar100_transfer_analysis_6layerCNN}
\centering
\resizebox{\linewidth}{!}{
\begin{tabular}{lcccc}
\toprule
Model & (SL) 6-layer-CNN & (SL) 6-layer-CNN-adv & (RL) 6-layer-CNN & (RL) 6-layer-CNN-adv \\
\hline
(SL) 6-layer-CNN & 0.03 & 0.40 & 8.27 & 16.59 \\
(SL) 6-layer-CNN-adv & 1.32 & 0.07 & 9.09 & 17.78 \\
(RL) 6-layer-CNN & 1.82 & 1.30 & 8.04 & 15.89 \\
(RL) 6-layer-CNN-adv & 20.17 & 14.52 & 13.08 & 18.16 \\
\hline
(SL) 4-layer-CNN & 30.50 & 27.22 & 30.70 & 29.60 \\
(SL) 4-layer-CNN-adv & 33.34 & 30.23 & 33.76 & 30.53 \\
(RL) 4-layer-CNN & 21.63 & 20.61 & 21.83 & 21.16 \\
(RL) 4-layer-CNN-adv & 21.50 & 19.63 & 21.87 & 20.35 \\
\hline
\end{tabular}%
}
\end{table}

\begin{table}[htbp]
\caption{Transfer analysis in CIFAR-100 on ResNet-18 (non-pretrained), evaluation under \attacksetting.}
\label{tab:cifar100_transfer_resnet18_nopt}
\centering
\resizebox{\linewidth}{!}{%
\begin{tabular}{lcccc}
\toprule
Model & (SL) Resnet18 & (SL) Resnet18-adv & (RL) Resnet18 & (RL) Resnet18-adv \\
\hline
(SL) Resnet18        & 14.83 & 59.61 & 51.09 & 56.82 \\
(SL) Resnet18-adv    & 52.08 & 12.61 & 48.84 & 54.36 \\
(RL) Resnet18        & 51.47 & 57.22 & 32.15 & 48.92 \\
(RL) Resnet18-adv    & 65.37 & 65.99 & 60.73 & 29.51 \\
\bottomrule
\end{tabular}
}
\end{table}

\begin{table}[htbp]
\caption{Transfer analysis in CIFAR-100 on ResNet-18 (pretrained), evaluation under \attacksetting.}
\label{tab:cifar100_transfer_resnet18_pt}
\centering
\resizebox{\linewidth}{!}{%
\begin{tabular}{lcccc}
\toprule
Model & (SL) Resnet18-pt & (SL) Resnet18-pt-adv & (RL) Resnet18-pt & (RL) Resnet18-pt-adv \\
\hline
(SL) Resnet18-pt     & 30.65 & 31.36 & 41.02 & 56.81 \\
(SL) Resnet18-pt-adv & 27.70 & 26.55 & 41.23 & 56.95 \\
(RL) Resnet18-pt     & 47.20 & 47.26 & 45.91 & 58.98 \\
(RL) Resnet18-pt-adv & 71.38 & 69.74 & 68.36 & 41.93 \\
\bottomrule
\end{tabular}
}
\end{table}

\subsection{ImageNet-100}
The complete performance on ImageNet-100 across a 4-layer CNN, 6-layer CNN, and ResNet-18 is shown in Table~\ref{tab:imagenet100_model_performance_across_models}. Since our main analyses focus on CIFAR-10/100, we treat ImageNet-100 as a supplementary scale check for verification; therefore, the ImageNet-100 results do not contain transfer analysis. On ImageNet-100, RL-adv exhibits a notable clean accuracy drop (45.92\% vs.\ 57.64\% for SL) with only marginal robustness gains over plain RL. This suggests that the current adversarial training configuration (designed primarily for CIFAR-scale experiments) may require task-specific tuning for higher-resolution, more complex datasets.

\begin{table}[htbp]
\caption{Model robustness evaluation (train, test, and AEs) in ImageNet-100 datasets across models, evaluation under \attacksetting.}
\label{tab:imagenet100_model_performance_across_models}
\begin{center}
\begin{tabular}{lcc}
\toprule
Model & Clean test (\%) & AE (\%) \\
\hline
(SL) 4-layer-CNN & 48.58 & 2.80 \\
(SL) 4-layer-CNN-adv & 45.76 & 2.92 \\
(RL) 4-layer-CNN & 47.88 & 10.20 \\
(RL) 4-layer-CNN-adv & 47.56 & 11.06 \\
\hline
(SL) 6-layer-CNN & 57.64 & 5.72 \\
(SL) 6-layer-CNN-adv & 58.00 & 5.36 \\
(RL) 6-layer-CNN & 55.60 & 18.04 \\
(RL) 6-layer-CNN-adv & 45.92 & 18.24 \\
\hline
(SL) Resnet18 & 74.00 & 45.00 \\
(SL) Resnet18-adv & 74.90 & 42.62 \\
(RL) Resnet18 & 73.92 & 49.02 \\
(RL) Resnet18-adv & 65.28 & 42.96 \\
\hline
\end{tabular}
\end{center}
\end{table}

\section{AutoAttack Evaluation on CNN Backbones}\label{app:aa_results}
In addition to evaluating robustness under standard PGD attacks, we further adopt AutoAttack as a complementary and more reliable assessment of adversarial robustness. For experiments on CIFAR-10, CIFAR-100, and ImageNet-100, we set the maximum perturbation budgets for AutoAttack to \(\epsilon\) = 7.0, 7.0, and 3.5, respectively, using the standard AutoAttack configuration. The full AutoAttack evaluation results for the 6-layer CNN on CIFAR-100 and ImageNet-100 are provided in Table~\ref{tab:aa_result_6layercnn_cifar100} and Table~\ref{tab:aa_result_6layercnn_imagenet100}. On CIFAR-10, RL-adv achieves a substantial margin over SL-adv across all AutoAttack components. On CIFAR-100 and ImageNet-100, the RL-adv advantage is smaller and, under APGD-CE on CIFAR-100, SL-adv slightly outperforms RL-adv (19.02\% vs.\ 17.95\%). This reduced margin on more complex datasets is consistent with the observation that RL's gradient disruption effect is most pronounced when model capacity is sufficient relative to task complexity (see also the clean accuracy discussion in Appendix~\ref{app:cifar_100}). Tuning adversarial training hyperparameters separately for each dataset-paradigm combination may recover a larger margin, but we leave this to future work to maintain a fair comparison with identical configurations across all settings.

\begin{table}[htbp]
\centering
\caption{AutoAttack results for CNN on CIFAR-100 dataset.}
\label{tab:aa_result_6layercnn_cifar100}
\resizebox{\linewidth}{!}{
\begin{tabular}{lccccc}
\toprule
              & Clean (\%) & APGD-CE (\%) & APGD-T (\%) & FAB-T (\%) & SQUARE (\%) \\
\midrule
6-layer-CNN-SL      & 64.75 & 4.75    & 3.91   & 3.91  & 3.48   \\
6-layer-CNN-SL-adv & 63.30  & 19.02   & 15.91  & 15.91 & 13.95  \\
6-layer-CNN-RL       & 59.80  & 4.19    & 3.44   & 3.44  & 3.04   \\
6-layer-CNN-RL-adv & 56.54 & 17.95 & 17.31 & 17.31 & 15.14 \\
\bottomrule
\end{tabular}
}
\end{table}

\begin{table}[htbp]
\centering
\caption{AutoAttack results for CNN on ImageNet-100 dataset.}
\label{tab:aa_result_6layercnn_imagenet100}
\resizebox{\linewidth}{!}{
\begin{tabular}{lccccc}
\toprule
              & Clean (\%) & APGD-CE (\%) & APGD-T (\%) & FAB-T (\%) & SQUARE (\%) \\
\midrule
6-layer-CNN-SL      & 57.64 & 5.64    & 4.24   & 4.24  & 3.62   \\
6-layer-CNN-SL-adv & 58.00 & 11.54   & 9.34   & 9.34  & 8.08    \\
6-layer-CNN-RL       & 55.60  & 4.96    & 3.52   & 3.52  & 3.10     \\
6-layer-CNN-RL-adv & 45.92 & 13.08 & 10.72 & 10.72 & 8.94 \\
\bottomrule
\end{tabular}
}
\end{table}

% \subsection{Pareto Analysis: Clean-Robust Trade-off}\label{app:pareto}
% To investigate the trade-off between clean accuracy and adversarial robustness, we conduct a trade-off evaluation using a 6-layer CNN architecture on CIFAR-10. We compare Supervised Learning (SL) and Reinforcement Learning (RL) under identical training settings, including a learning rate of $1 \times 10^{-4}$ and a batch size of 256. To account for optimization stochasticity, we train both RL-adv and SL with three independent random seeds $\{10, 20, 42\}$, while the seed for the standard PGD attack is fixed to 33. We report the mean clean accuracy and robust accuracy (mean $\pm$ standard deviation across seeds). Robustness is evaluated using a standard PGD, ensuring a consistent evaluation protocol.

\section{Proofs of Theoretical Statements}\label{app:theory_proofs}
\subsection{Proof of Proposition~\ref{prop:loss_growth}}
\begin{proof}
By Assumption~\ref{assump:loss_smoothness}, $J_M(\cdot,y)$ is $L_J$-smooth on $C$. Therefore, for each PGD step,
\[
J_M(x_{t+1},y)
\le
J_M(x_t,y)
+
\inner{\nabla_x J_M(x_t,y)}{x_{t+1}-x_t}
+
\frac{L_J}{2}\norm{x_{t+1}-x_t}_2^2.
\]
Using $g_t=\nabla_xJ_M(x_t,y)$ and $s_t=x_{t+1}-x_t$, this becomes
\[
J_M(x_{t+1},y)-J_M(x_t,y)
\le
\inner{g_t}{s_t}
+
\frac{L_J}{2}\norm{s_t}_2^2.
\]
We next bound the realized step length. Since $x_t\in C$, we have $\Pi_C(x_t)=x_t$. Euclidean projection onto a closed convex set is non-expansive, meaning that for any $a,b$,
\[
\norm{\Pi_C(a)-\Pi_C(b)}_2
\le
\norm{a-b}_2.
\]
Taking $a=x_t+\alpha d_t$ and $b=x_t$ gives
\[
\norm{s_t}_2
=
\norm{x_{t+1}-x_t}_2
=
\norm{\Pi_C(x_t+\alpha d_t)-\Pi_C(x_t)}_2
\le
\norm{\alpha d_t}_2.
\]
By construction, $\norm{d_t}_2=1$ when $g_t\ne 0$ and $\norm{d_t}_2=0$ when $g_t=0$. Hence $\norm{d_t}_2\le 1$, and therefore
\[
\norm{s_t}_2=\ell_t\le \alpha.
\]
By Cauchy--Schwarz,
\[
\inner{g_t}{s_t}
\le
\norm{g_t}_2\norm{s_t}_2
\le
\alpha\norm{g_t}_2.
\]
Moreover,
\[
\frac{L_J}{2}\norm{s_t}_2^2
\le
\frac{L_J\alpha^2}{2}.
\]
Combining these inequalities yields the one-step bound
\[
J_M(x_{t+1},y)-J_M(x_t,y)
\le
\alpha\norm{g_t}_2
+
\frac{L_J\alpha^2}{2}.
\]
Summing from $t=0$ to $\tau-1$ gives
\[
J_M(x_\tau,y)-J_M(x_0,y)
\le
\alpha\sum_{t=0}^{\tau-1}\norm{g_t}_2
+
\frac{L_J\alpha^2}{2}\tau.
\]
Using the definition
\[
\mathcal G_\tau(M;x_0,y)
=
\sum_{t=0}^{\tau-1}\norm{g_t}_2
\]
gives the equivalent form
\[
J_M(x_\tau,y)
\le
J_M(x_0,y)
+
\alpha\mathcal G_\tau(M;x_0,y)
+
\frac{L_J\alpha^2}{2}\tau.
\]
\end{proof}

\subsection{Proof of Proposition~\ref{prop:step_interference}}
\begin{proof}
The first identity follows by telescoping:
\[
x_\tau-x_0
=
(x_\tau-x_{\tau-1})+(x_{\tau-1}-x_{\tau-2})+\cdots+(x_1-x_0)
=
\sum_{t=0}^{\tau-1}s_t.
\]
Taking squared norms gives
\[
\norm{x_\tau-x_0}_2^2
=
\left\|\sum_{t=0}^{\tau-1}s_t\right\|_2^2.
\]
Expanding the squared norm,
\[
\left\|\sum_{t=0}^{\tau-1}s_t\right\|_2^2
=
\sum_{t=0}^{\tau-1}\norm{s_t}_2^2
+
2\sum_{0\le i<j\le \tau-1}\inner{s_i}{s_j}.
\]
Since $s_t=\ell_tu_t$ whenever $\ell_t>0$, and the identity is unchanged when $\ell_t=0$ because all terms involving that step vanish,
\[
\norm{s_t}_2^2=\ell_t^2,
\qquad
\inner{s_i}{s_j}=\ell_i\ell_j\inner{u_i}{u_j}.
\]
Substituting these expressions yields
\[
\norm{x_\tau-x_0}_2^2
=
\sum_{t=0}^{\tau-1}\ell_t^2
+
2\sum_{0\le i<j\le \tau-1}\ell_i\ell_j\inner{u_i}{u_j}.
\]
Since $R_\tau(M;x_0,y)=\norm{x_\tau-x_0}_2$, the result follows.
\end{proof}

\subsection{Proof of Theorem~\ref{thm:finite_budget_attackability}}
\begin{proof}
The attack-loss growth bound is exactly Proposition~\ref{prop:loss_growth}. For the margin-preservation bound, Assumption~\ref{assump:local_margin} gives, for any $a,b\in C$,
\[
|m_M(a,y)-m_M(b,y)|\le L_m\norm{a-b}_2.
\]
Taking $a=x_\tau$ and $b=x_0$ gives
\[
m_M(x_\tau,y)
\ge
m_M(x_0,y)-L_m\norm{x_\tau-x_0}_2.
\]
By definition, $R_\tau(M;x_0,y)=\norm{x_\tau-x_0}_2$, so
\[
m_M(x_\tau,y)
\ge
m_M(x_0,y)-L_mR_\tau(M;x_0,y).
\]
The expansion of $R_\tau(M;x_0,y)^2$ follows from Proposition~\ref{prop:step_interference}. This proves the theorem.
\end{proof}

% \section{LLM Usage Disclosure}\label{app:llm_disclosure}
% We used Large Language Models (LLMs) (shown in Table~\ref{tab:llm_disclosure}) \textbf{only for English-language polishing}, including grammar correction and wording suggestions. \textbf{No new scientific contents, including equations, experimental designs, or codes} were generated by the models. We did not provide any non-public data or code to the models. All model suggestions were manually reviewed and edited by the authors for technical correctness. The authors take full responsibility for the final content; the  LLM is not an author.

% \begin{table}[H]
% \centering
% \caption{LLM usage disclosure.}
% \label{tab:llm_disclosure}
% \begin{tabular}{lcc}
% \toprule
% Model & Version & Access date \\
% \midrule
% ChatGPT & GPT-5.5            & 2026.04--2026.04 \\
% Claude  & Opus 4.6           & 2026.03--2026.04 \\
% ChatGPT & GPT-5.4 (Thinking) & 2026.02--2026.04 \\
% \bottomrule
% \end{tabular}
% \end{table}

\end{document}